\documentclass{article}

\usepackage{float}
\usepackage{microtype}
\usepackage{graphicx}
\usepackage{subfigure}
\usepackage{booktabs} 
\usepackage{wrapfig}
\usepackage{amsmath}
\usepackage{graphicx}
\usepackage{hyperref}
\usepackage{url}
\usepackage{bbm}
\usepackage{multicol}
\usepackage{multirow}
\usepackage{mathtools}
\usepackage{algorithm}
\usepackage{algorithmic}
\floatname{algorithm}{Procedure}
\renewcommand{\algorithmicrequire}{\textbf{Input:}}
\renewcommand{\algorithmicensure}{\textbf{Output:}}

\usepackage{amsmath,amsfonts,bm,amsthm}

\def\eqref#1{equation~\ref{#1}}

\def\1{\bm{1}}

\def\vx{{\bm{x}}}

\DeclareMathAlphabet{\mathsfit}{\encodingdefault}{\sfdefault}{m}{sl}
\SetMathAlphabet{\mathsfit}{bold}{\encodingdefault}{\sfdefault}{bx}{n}

\def\gX{{\mathcal{X}}}
\def\gY{{\mathcal{Y}}}

\def\sI{{\mathbf{1}}}

\DeclareMathOperator*{\argmax}{arg\,max}
\DeclareMathOperator*{\argmin}{arg\,min}

\newcommand{\ynoise}{\widetilde{y}}

\newcommand{\etanoise}{\widetilde{\eta}}

\newcommand{\loss}{L}

\newtheorem{theorem}{Theorem}
 
\newtheorem{lemma}{Lemma}
\newtheorem{definition}{Definition}
\newtheorem{assumption}{Assumption}

\newtheorem{remark}{Remark}

\DeclareMathOperator{\LR}{LR}

\DeclareMathOperator*{\pr}{Pr}

\newcommand{\AdaCorr}{\texttt{AdaCorr}}
\newcommand{\LRTCorr}{\texttt{LRT-Correction}}

\usepackage{enumitem}
  
\usepackage{hyperref}

\usepackage[accepted]{icml2020}

\icmltitlerunning{Error-Bounded Correction of Noisy Labels}

\begin{document}\hfuzz=100pt

\twocolumn[
\icmltitle{Error-Bounded Correction of Noisy Labels}




\begin{icmlauthorlist}
\icmlauthor{Songzhu Zheng}{sbu-ams}
\icmlauthor{Pengxiang Wu}{rutgers}
\icmlauthor{Aman Goswami}{bain}
\icmlauthor{Mayank Goswami}{cuny}
\icmlauthor{Dimitris Metaxas}{rutgers}
\icmlauthor{Chao Chen}{sbu-bmi}
\end{icmlauthorlist}

\icmlaffiliation{sbu-ams}{Department of Applied Mathematics and Statistics, Stony Brook University, NY, USA}
\icmlaffiliation{sbu-bmi}{Department of Biomedical Informatics, Stony Brook University, NY, USA}
\icmlaffiliation{bain}{Bain \& Company, Bangalore, India.}
\icmlaffiliation{rutgers}{Department of Computer Science, Rutgers University, NJ, USA}
\icmlaffiliation{cuny}{Department of Computer Science, City University of New York, NY, USA}

\icmlcorrespondingauthor{Songzhu Zheng}{zheng.songzhu@stonybrook.edu}

\icmlkeywords{Supervised Learning, Label Noisy}

\vskip 0.3in
]



\printAffiliationsAndNotice{}  

\begin{abstract}
To collect large scale annotated data, it is inevitable to introduce label noise, i.e., incorrect class labels. To be robust against label noise, many successful methods rely on the noisy classifiers (i.e., models trained on the noisy training data) to determine whether a label is trustworthy. However, it remains unknown why this heuristic works well in practice. In this paper, we provide the first theoretical explanation for these methods. We prove that the prediction of a noisy classifier can indeed be a good indicator of whether the label of a training data is clean. Based on the theoretical result, we propose a novel algorithm that corrects the labels based on the noisy classifier prediction. The corrected labels are consistent with the true Bayesian optimal classifier with high probability. We incorporate our label correction algorithm into the training of deep neural networks and train models that achieve superior testing performance on multiple public datasets.
\end{abstract}

\section{Introduction}

Label noise is ubiquitous in real world data. It may be caused by  unintentional mistakes of manual or automatic annotators 
\citep{yan:2014, Veit:2017}.
It may also be introduced by malicious attackers \citep{Steinhardt_DataPoison_NIPS2017}. 
Noisy labels impair the performance of a model \citep{Smyth:1994,Brodley:1999}, especially a deep neural network, which tends to have strong memorization power \citep{Frenay:2014,Zhang_noise_ICLR2017}.
Improving the robustness of a model trained with noisy labels is a crucial yet challenging task in many applications \citep{Mnih_Aerial_Image_ICML2012,Wu_Face_Noise_IEEETrans2018}. 

Many methods have been proposed to train a robust model on data with label noise. 
One may re-calibrate the model by explicitly estimating a noise \emph{transition matrix}, namely, the probability of one label being corrupted into another \citep{goldberger_Adaptation_ICLR2017, Patrini:2017}. 
One may also introduce hidden layers \citep{reed2014training}, prior on data distribution \citep{lee2019robust} or modified loss function \cite{Rooyen:2015, shen2019trim, zhang2018general_loss} to improve the robustness of the model. 
However, these methods either assume strong global priors on the data or lack sufficient supervision for the neural network to achieve satisfying performance. Furthermore, global model-correction mechanisms tend to rely on a few parameters; estimating these parameters can be challenging and the error will lead to failing of the training. 

To adapt to heterogeneous noise pattern and to fully exploit the power of deep neural networks, \emph{data-re-calibrating} methods have been proposed to focus on individual data instead of an overall model adjustment \citep{Malach:2017,jiang:2018,Han:2018, tanaka2018jointlearn, wang2018openset, ren2018reweight, cheng2020learning}. 
These methods learn to re-calibrate the model on each individual datum depending on its own context. 
They gradually collect clean data whose labels are trustworthy. 
As more clean data are collected, the quality of the trained models improves.
These methods slowly accumulate useful/trustworthy information and eventually attain state-of-the-art quality models.

Despite the success of data-re-calibrating methods, their underlying mechanism remains elusive. It is unclear why the neural nets trained on noisy labels can help select clean data. A theoretical underpinning will not only explain the phenomenon, but also advance the methodology. One major challenge for these methods is to control the data re-calibration quality. 
It is hard to monitor the model's re-calibrating decision on individual data. An aggressive selection of clean data can unknowingly accumulate irreversible errors. On the other hand, an overly-conservative strategy can be very slow in training, or stops with insufficient clean data and mediocre models. A theoretical guarantee will help develop models with self-assurance that the decision on each datum is reasonably close to the truth.

In this paper, we provide \emph{the first theoretical explanation for data-re-calibrating methods}. Our main theorem states that a noisy classifier (i.e., one trained on noisy labels) can identify whether a label has been corrupted. In particular, we prove that when the noisy classifier has low confidence on the label of a datum, such label is likely corrupted. In fact, we can quantify the \textit{threshold of confidence}, below which the label is likely to be corrupted, and above which is it likely to be not. We also empirically show that the bound in our theorem is tight.

Our theoretical result not only explains existing data-re-calibrating methods, but also suggests a new solution for the problem. As a second contribution of this paper, we propose a novel method for noisy-labeled data. Based on our theorem and statistical principles, we verify the purity of a label through a likelihood ratio test w.r.t.~the prediction of a noisy classifier, and the threshold value of confidence. The label is corrected or left intact depending on the test result. We prove that this simple label-correction algorithm has a guaranteed success rate and will recover the true labels with high probability. We incorporate the label-correction algorithm into the training of deep neural networks. We validate our method on different datasets with various noise patterns and levels. 
Our theoretically-founded method outperforms state-of-the-arts due to its simplicity and due to its principled design. 

Our paper shows that a theorem that is well-grounded in applications will inspire elegant and powerful algorithms even in deep learning settings. Our contribution is two-fold:
\begin{itemize}[topsep=0pt, parsep=2pt, itemsep = 0pt,partopsep=0pt]
    \item We provide a theorem quantifying how a noisy classifier's prediction correlates to the purity of a datum's label. This provides theoretical explanation for data-re-calibrating methods for noisy labels. 
    \item Inspired by the theorem, we propose a new label-correction algorithm with guaranteed success rate. We train neural networks using the new algorithm and achieve superior performance.
\end{itemize}

The code of this paper can be found in \url{https://github.com/pingqingsheng/LRT.git}.

\subsection{Related Work}
One representative strategy for handling label noise is to model and employ noise transition matrix to correct the loss. For example,
\citet{Patrini:2017} propose to correct the loss function with estimated noise pattern. The resulting loss is an unbiased estimator of the ground truth loss, and enables the trained model to achieve better performance. However, 
such an estimator relies on strong assumptions and could be inaccurate in certain scenarios.
\citet{reed2014training} consider modeling the noise pattern with a hidden layer. The learning of this hidden layer is regularized with a feature reconstruction loss, yet without a guarantee that the true label distribution is learned.
Another method mentioned in their work is to minimize the entropy of neural network output;
however, this method tends to predict a single class.
To address this weakness, 
\citet{Hendrycks:19} propose to utilize a small number of trusted, clean data to pre-train a network and estimate the noise pattern. 
However, such clean data may not always be available in practice.

Alternatively, another direction proposes to design models that are intrinsically robust to noisy data. \citet{Crammer:2009} introduce a regularized confidence weighting learning algorithm (AROW),
which attempts to preserve the weight distribution as much as possible while requiring the model to maintain discrimination ability. The follow-up work (\citealt{Crammer:2010}) improves this algorithm by herding the updating direction via specific velocity field (NHERD), and achieves better performance. Both of these works impose constraints on parameters, which, however, could prevent classifiers from adapting to complex datasets.
Another similar strategy proposes to assume Gaussian distribution for features, and models the data with a robust generative classifier \cite{lee2019robust}. However, such an assumption may not generalize to other complex scenarios.

\citet{Arpit_Memorization_ICML2017} show that 
deep neural networks tend to learn meaningful patterns before they over-fit to noisy ones. Based on this observation, they propose to add Gaussian or adversarial noise to input when training with noisy labels, and empirically show that such data perturbation is able to make the resulting model more robust. Other commonly adopted techniques, such as weight decay and dropout, are also shown to be effective in increasing the robustness of trained classifier (\citealt{Arpit_Memorization_ICML2017}; \citealt{Zhang_noise_ICLR2017}). However, the intrinsic reasons for this phenomenon still remain unclear and overfitting to noisy label is extremely likely.
Data-re-calibrating methods select clean data while eliminating noisy ones during training. For example, \citet{Malach:2017} and \citet{Han:2018} train two networks simultaneously, and update the networks only with samples that are considered clean by both networks. Similarly, \citet{jiang:2018} also use two networks: the first one is pre-trained to learn a curriculum, and then utilized to select clean samples for training the second network. These methods deliver promising results but lack control of the quality of the collected clean data.

Finally, beyond deep learning framework, there are several theoretic works that demonstrate the robustness of a variety of losses to label noise (\citealt{long_2010_random_JMLR}; \citealt{natarajan:2013}; \citealt{Ghosh_riskmini_2015_NIPS}; \citealt{Rooyen:2015}).
 Following the work of (\citealt{Yizhen:2018}), \citet{gao:2016} propose an algorithm that can converge to the Bayesian optimal classifier under different noise settings. Moreover, they provide in-depth discussion regarding the performance of k-nearest neighbor (KNN) classifiers. However, the problem with KNN is that it is computationally intensive and difficult to be incorporated into a learning context. Within the framework of deep learning, there are more efforts that need to be made to bridge theory and practice.

\section{The Main Theorem: Probing Label Purity Using the Noisy Classifier}

Our main theorem answers the following question: without knowing the ground truth, how to decide whether a label is corrupted or not.
During training, the only information one can rely on is a noisy classifier, i.e., one that is trained on the corrupted labels. Data-re-calibrating methods use the noisy classifier to decide whether a datum is clean-labeled. However, these methods lack a theoretical justification. 

We establish the relationship between a noisy classifier and the purity of a label. We prove that if the classifier has low confidence on a datum with regard to its current label, then this label is likely corrupted. This result provides the first theoretical explanation of why noisy classifiers can be used to determine the purity of labels in previous methods. 

This section is organized as follows. We start by providing basic notations and assumptions. Next, we state the main theorem for binary classification and then extend it to the multiclass setting. We also use experiments on synthetic data and CIFAR10 to validate the tightness of our bound.

\subsection{Preliminaries and Assumptions}
We first focus on binary classification. Later the result will be extended to multiclass setting. Let $\mathcal{X}$ be the feature space, $\mathcal{Y}=\{0,1\}$ be the label space.  
The joint probability distribution, $D$, can be factored as $D(\vx,y) = \pr(y|\vx)\pr(\vx)$. We denote by $\eta(\vx) = \pr(y=1|\vx)$ the \emph{true conditional probability}. 
The \emph{risk} of a binary classifier $h:\gX\rightarrow \gY$ is $R(D,h) = \Pr_{(\vx,y) \sim D} [h(\vx) \neq y]$.
A \emph{Bayes optimal classifier} is the minimizer of the risk over all possible hypotheses, i.e., \mbox{$h^\ast = \argmin_h R(D,h)$}. 
It can be calculated using the true conditional probability, $\eta$,
\begin{align*}
    h^*(\vx) = 
    \sI_{ \left\{\eta(\vx)>\frac{1}{2} \right\}}(\mathbf{x}) := 
    \left\{
        \begin{array}{cl}
        1 & \eta(\vx) > \frac{1}{2}\\
        0 & \text{otherwise}
        \end{array}.
    \right.
\end{align*}

We assume $\eta$ satisfies the Tsybakov condition 
\citep{tsybakov2004optimal}. This condition, also called margin assumption, stipulates that the uncertainty of $\eta$ is bounded. In other words, the margin region close to the decision boundary, $\{\vx\in \mathcal{X} \mid \eta(\vx)=1/2\}$, has a bounded volume.

\begin{assumption}[Tsybakov Condition] There exist constants $C,\lambda>0$, and $ t_0 \in (0, \frac{1}{2}]$, such that for all $t \leq t_0$, 
$$\Pr\left[ \left| \eta(\vx) - \frac{1}{2} \right| \leq t \right] \leq Ct^\lambda.$$ 
\end{assumption}

This assumption is adopted in previous works such as \citep{chaud:2014, belkin:2018, qiao:2019}. However, we have not seen any empirical verification of the condition in real datasets. In this paper, we conduct experiments to verify this condition and provide empirical estimation of the constants $C$ and $\lambda$. Our experiments indicate that this condition holds with moderate values of the constants $C$ and $\lambda$. 

\textbf{The noisy label setting.}
Instead of samples from $D$, we are given a sample set with noisy labels
$\mathcal{S}=\{(\vx,\widetilde{y})\}$, where $\ynoise$ is the possibly corrupted label based on the true label $y$. We assume a \emph{transition probability} \mbox{$\tau_{i\rightarrow j} = \pr(\widetilde{y}=j | y=i)$}, i.e., the chance a true label $y$ is flipped from class $i$ to class $j$. For simplicity, we denote $\tau_{ij}=\tau_{i\rightarrow j}$. The transition probabilities $\tau_{01}$ and $\tau_{10}$ are independent of the true joint distribution $D$ and the feature $\vx$.
We denote the conditional probability of the noisy labels as $\widetilde{\eta}(\vx) = \pr(\widetilde{y}=1 | \vx)$. We call $\widetilde{\eta}$ the \emph{noisy conditional probability}. It is easy to verify that $\etanoise$ is linear to the true conditional probability, $\eta$: 
\begin{align*}
\widetilde{\eta}(\vx) 
&= (1-\tau_{10})\eta(\vx) + \tau_{01}[1-\eta(\vx)] \nonumber\\
&= (1-\tau_{01}-\tau_{10})\eta(\vx) + \tau_{01}.
\label{eq:etanoise}
\end{align*}

We intend to learn a classifier whose prediction is consistent with the Bayes optimal classifier $h^{\ast}$. Therefore, we call the prediction of $h^{\ast}$ the \emph{correct label}. 
\begin{definition}[Correct Label]
Given $\vx$, its \emph{correct label} is the Bayes optimal classifier prediction $h^{\ast}(\vx)$. 
\end{definition}
The correct label, $h^{\ast}(\vx)$, is subtly different from the true label, $y$. In particular, $h^{\ast}(\vx)$ is uniquely decided by $\eta(\vx)$, whereas $y$ is a sample from $\eta$. Since $h^\ast$ is our final goal, we focus on recovering the correct label, $h^{\ast}(\vx)$, instead of $y$. 

\subsection{The Main Theorem}

Our main theorem connects a noisy classifier \mbox{$f:\mathcal{X}\rightarrow \mathcal{Y}$} with the chance of a noisy label $\ynoise$ being correct. We assume $f$ is trained on the noisy labels and is trained well enough, i.e., $\epsilon$-close to the noisy conditional probability, $\etanoise$. For convenience, we denote by $f_{\ynoise}$ the classifier prediction of label being $\ynoise$, formally, $f_{\ynoise}(\vx) = f(\vx)$ if $\ynoise = 1$, and $1-f(\vx)$ otherwise. Define the estimation error $\epsilon := \| f - \etanoise \|_{\infty}$.

\begin{theorem}
\label{thm:main}
\label{epsilontheorem}

Assume $\eta(\vx)$ satisfies the Tsybakov condition with constants $C,\lambda>0$, and $t_0 \in (0,\frac{1}{2}]$. Assume $\epsilon \leq t_0(1-\tau_{10}-\tau_{01}).$ For $\Delta=\frac{1-|\tau_{10}-\tau_{01}|}{2}$, we have:
\[
\pr\nolimits_{(x,y)\sim D} \Big[\ynoise = h^*(\vx) , f_{\ynoise}(\vx) < \Delta\Big] \leq C \Big[ O(\epsilon) \Big]^{\lambda}.
\]
\end{theorem}

\textbf{Implication of the theorem.} 
Intuitively, the theorem states that a noisy label $\ynoise$ has bounded probability to be correct if it has a low vote-of-confidence by $f$. The upper bound of the probability is controlled by $\epsilon$, the approximation error of $f$.  In other words, the better $f$ approximates $\etanoise$, the tighter the bound is. This justifies the usage of a good-quality $f$ to determine if $\ynoise$ is trustworthy. Later we will show $\epsilon$ is reasonably small in deep learning setting and the bound is tight in practice. 

We remark that the constant $\Delta$ and the constant hidden inside the big-O in the theorem depend on $\tau_{ij}$'s, which are unknown in practice. 
Based on this theorem, we will propose a new  label-correction algorithm that determines $\Delta$ robustly in practice without knowing $\tau_{ij}$'s. 


\subsubsection{Proof of Theorem~\ref{thm:main}}

{\textbf{Preliminary Lemmata.}} To prove this theorem, we need to first prove two lemmata. Lemma~\ref{lem:idealdelta} will show that if a classifier $g$ is a linear transformation of $\eta$, when the value $g_{\ynoise}$ is below a certain threshold, $\ynoise$ is unlikely to be consistent with the true Bayesian optimal decision, $h^\ast$. Next, Lemma~\ref{lem:noisyeta} states that since $\widetilde{\eta}(\vx)$ is a linear transformation of $\eta(\vx)$, Lemma~\ref{lem:idealdelta} will apply to $\widetilde{\eta}(\vx)$ and $\Delta$ can be set accordingly. Finally, based on the conclusion of Lemma~\ref{lem:noisyeta} and the Tsybakov condition, we can upperbound $\pr\left[\widetilde{y}=h^*(\vx), f_{\ynoise}(\vx)<\Delta\right]$ if $f$ is $\epsilon$-close to $\etanoise$.

\begin{lemma}\label{lem:idealdelta}
If a classifier $g$ depends linearly on $\eta$, i.e., $g(\vx) = a \eta(\vx) + b$ with $a,b> 0$. Set $\Delta = \min\left(\frac{a}{2}+b, 1-b-\frac{a}{2}\right)$. We have
\begin{equation}
\label{eq:lemma1}
\Pr\nolimits_{(x,y)\sim D}\Big[\widetilde{y}=h^*(\vx), g_{\ynoise}(\vx)<\Delta\Big] = 0
\end{equation}
\end{lemma}
 
\begin{proof}
To calculate $\Pr\nolimits_{(x,y)\sim D}\Big[\widetilde{y}=h^*(\vx), g_{\ynoise}(\vx)<\Delta\Big]$, we enumerate two cases: 
\item{Case 1:} $\ynoise = 1$. 
Observe $h^{\ast}(\vx)=1$ iff $\eta(\vx) > 1/2$;  $g_{\ynoise}(\vx)=g(\vx) = a\eta(\vx)+b<\Delta$ iff $\eta(\vx) < \frac{\Delta - b}{a}$. We have:
\begin{equation}
\Pr\Big[\ynoise=h^*(\vx), g_{\ynoise}(\vx) < \Delta\Big]
= \Pr\left[\frac{1}{2} < \eta(\vx) < \frac{\Delta-b}{a}\right].
\label{eq:prob1}
\end{equation}
We next show that this probability is 0 for the chosen $\Delta= \min\left(\frac{a}{2}+b, 1-b-\frac{a}{2}\right)$. If  $\Delta = \frac{a}{2} + b$, the probability is zero as $\frac{\Delta-b}{a} = \frac{1}{2}$. Otherwise, $\Delta = 1-b-\frac{a}{2}$. 
We know that $1-b-\frac{a}{2}< \frac{a}{2} + b$. Therefore, $1-2b<a$. In this case, 
\[\frac{\Delta-b}{a} = \frac{1-2b}{a} - \frac{1}{2} < 1-\frac{1}{2} = \frac{1}{2}.
\]
Thus we have $\Pr\left[\frac{1}{2}<\eta(\vx)<\frac{\Delta-b}{a}\right] =0$.

\item{Case 2:} $\ynoise =0$. Observe that $h^{\ast}(\vx) = 0$ iff $\eta(\vx)\leq 1/2$; $g_{\ynoise}(\vx) = 1-g(\vx) = 1-[a\eta(\vx)+b]<\Delta$ iff $\eta(\vx) > L:= \frac{1-b-\Delta}{a}$, we have:
\begin{equation*}
\label{eq:prob2}
\Pr\Big[\ynoise=h^*(\vx), g_{\ynoise}(\vx) < \Delta\Big]
= \Pr\left[{L} < \eta(\vx) < \frac{1}{2}\right].
\end{equation*}
Similar to Case 1, by checking when $\Delta = \frac{a}{2} + b$ and when $\Delta = 1-b-\frac{a}{2}$, we can verify that $\Pr\left[\frac{1-b-\Delta}{a} < \eta(\vx) < \frac{1}{2}\right]=0$.

This proves Equation (\ref{eq:lemma1}) and completes the proof.
\end{proof}

\begin{lemma}\label{lem:noisyeta}
Let $\Delta = \frac{1-|\tau_{10}-\tau_{01}|}{2}$. Let $\widetilde{\eta}_1 = \widetilde{\eta}$ and $ \widetilde{\eta}_0 = 1- \widetilde{\eta}$.  


\begin{equation}
\label{eq:lemma2}
\Pr\nolimits_{(x,y)\sim D}\Big[\widetilde{y}=h^*(\vx), \widetilde{\eta}_{\ynoise}(\vx) < \Delta\Big]  = 0.
\end{equation}
\end{lemma}
\begin{proof} 
Recall $\widetilde{\eta}(\vx) = (1-\tau_{01}-\tau_{10})\eta(\vx) + \tau_{01}$, in which $\tau_{01}$ and $\tau_{10}$ are transition probabilities. We can directly prove this lemma using Lemma \ref{lem:idealdelta} by setting $g = \etanoise$ with $a = 1-\tau_{01}-\tau_{10}$ and $b = \tau_{01}$. 
\end{proof}

\begin{figure*}[!htb]
 \centering
 \begin{tabular}{ccc}
 \includegraphics[width=0.27\textwidth]{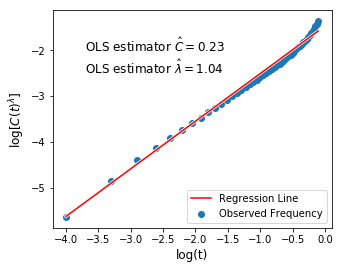} & 
 \includegraphics[width=0.25\textwidth]{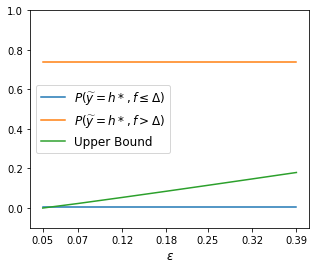} & 
 \includegraphics[width=0.25\textwidth]{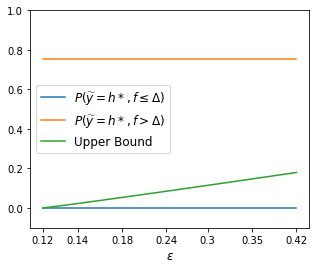} \\
 (a) & (b) & (c)
 \end{tabular}
 \caption{Synthetic experiment using CIFAR10 at noise level 20\%. (a): Check of Tsybakov condition using linear regression. Where y-axis is the proportion of data points at distance $t$ from decision boundary. (b): Proportion of labels that are not correct (not consistent with Bayes optimal decision rule) and the proposed upper bound. (c). Same as (b) but labels are corrupted with asymmetric noise. }
 \label{fig:cifar10_synthetic}
 \end{figure*}
 
\textbf{Proof of Theorem \ref{thm:main} using the Lemmata.} 
\begin{proof}
When $\ynoise = 1$, $f_{\ynoise}(\vx) = f(\vx) \geq \etanoise(\vx)-\epsilon$. 
\[{\small
\Pr\left[\widetilde{y}=h^*(\vx), f_{\widetilde{y}}(\vx) < \Delta\right] \leq \Pr\left[\widetilde{y}=h^*(\vx), \etanoise(\vx)-\epsilon < \Delta\right] 
}\]
Substituting $\Delta$ with $\Delta + \epsilon$ into equation~(\ref{eq:prob1}), we have:
\begin{align*}
& \Pr\left[\widetilde{y}=h^*(\vx)=1, \etanoise(\vx)-\epsilon < \Delta\right]\nonumber\\
= & \Pr\left[\widetilde{y}=h^*(\vx)=1, \etanoise(\vx)<  \Delta+\epsilon\right]\nonumber\\
= &
\Pr\left[\frac{1}{2} < \eta(\vx) < \frac{\Delta+\epsilon-\tau_{01}}{1-\tau_{01}-\tau_{10}}\right] \nonumber\\
\label{eq:thm1:eq2}
\end{align*}
Similar to Lemma \ref{lem:idealdelta}, by discussing the cases when $\Delta = \frac{1+\tau_{10}-\tau_{01}}{2}$ and when $\Delta= \frac{1+\tau_{01}-\tau_{10}}{2}$, we can show that  $\frac{\Delta-\tau_{01}}{1-\tau_{01}-\tau_{10}} < \frac{1}{2}$.
Based on the Tsybakov condition, we have
{\small
\begin{align*}
& \Pr\left[\frac{1}{2} < \eta(\vx) < \frac{\Delta-\tau_{01}}{1-\tau_{01}-\tau_{10}} + \frac{\epsilon}{1-\tau_{01}-\tau_{10}}\right]\nonumber\\ \leq &   
\Pr\left[\frac{1}{2} < \eta(\vx) < \frac{1}{2} + \frac{\epsilon}{1-\tau_{01}-\tau_{10}}\right] 
\leq C\left(\frac{\epsilon}{1-\tau_{01}-\tau_{10}}\right)^\lambda
\end{align*}}
This implies that: 
\begin{equation*}
    \Pr\left[\widetilde{y}=h^*(\vx)=1, f_{\widetilde{y}}(\vx) < \Delta\right] 
    \leq C\left(\frac{\epsilon}{1-\tau_{01}-\tau_{10}}\right)^\lambda
    \label{eq:thm1:eq4}
\end{equation*}

Similar to case 1 of Lemma~\ref{lem:idealdelta}, by using equation~(\ref{eq:lemma2}) for the case when $\ynoise = 0$, we can prove that 
\begin{align*}
& \Pr\left[\widetilde{y}=h^*(\vx)=0, f_{\widetilde{y}(\vx)} < \Delta\right] \nonumber\\
\leq &  \Pr\left[\widetilde{y}=h^*(\vx)=0, 1-\etanoise(\vx)-\epsilon < \Delta\right]\nonumber\\ 
= & \Pr\left[\frac{1-\tau_{01}-\Delta}{1-\tau_{10}-\tau_{01}}- \frac{\epsilon}{1-\tau_{10}-\tau_{01}} < \eta(\vx) < \frac{1}{2}\right] \nonumber\\
\leq &   
\Pr\left[ \frac{1}{2} - \frac{\epsilon}{1-\tau_{01}-\tau_{10}} < \eta(\vx) < \frac{1}{2}\right] \nonumber\\ 
\leq & C\left(\frac{\epsilon}{1-\tau_{01}-\tau_{10}}\right)^\lambda
\label{eq:thm1:eq5}
\end{align*}
Combining the two cases ($\ynoise=1$) and ($\ynoise=0$) completes the proof. 
\end{proof}

 
\begin{remark}
Indeed, we can also prove a bound for the opposite case: when $f_{\widetilde{y}}$ is highly confident, $\ynoise$ is correct with high probability. In this paper, we only focus on the bound in theorem~\ref{thm:main} as we only want to identify incorrect labels and fix them.
\end{remark}

\subsection{Multiclass Setting} 
Theorem \ref{thm:main} can be generalized to a multiclass setting. Let $\ynoise$ be the observed (possibly) corrupted label, $\eta_i(\vx) = \Pr(y=i \mid \vx)$ and $\widetilde{\eta}_i(\vx) = \Pr(\widetilde{y}=i \mid \vx)$. Recall $f_i(\vx)$ is the classifier's prediction on label $i$. Define $N_c$ to be the number of total classes and $[N_c] = \{1,2,\cdots,N_c\}$.

First we extend the Tsybakov condition to multiclass scenario \cite{chen:2006}. Denote by ${u_\vx}$ the Bayes optimal classifier prediction, or say the class predicted by $\eta(\vx)$, formally $ {u_\vx} := h^{\ast}(\vx) = \argmax_{i} \eta_i(\vx)$. Denote by ${s_\vx}$ the second best prediction, ${s_\vx}:=\argmax_{i\neq u_\vx} \eta_i(\vx)$. The difference between their corresponding true conditional probability is a non-negative function, whose zero level set $\{\vx|\eta_{{u_\vx}}(\vx)-\eta_{{s_\vx}}(\vx) = 0\}$ is the decision boundary of $h^*$. We assume the Tsybakov condition around the margin of this decision boundary: $\exists C,\lambda>0$ and $\exists t_0\in (0,1]$, such that for all $t\leq t_0$,
\begin{equation}
\Pr\big[\eta_{{u_\vx}}(\vx) - \eta_{{s_\vx}}(\vx) \leq t\big] \leq Ct^\lambda
\label{eq:multiclass-tsybakov}
\end{equation}

For any pair of labels $i,j \in [N_c]$, we have the linear relationship $\widetilde{\eta}_i(\vx) = \sum_{j \in [N_c]}\tau_{ji}\eta_j(\vx)$. Define $m_\vx := \argmax_i f_i(\vx)$. Define the estimation error $\epsilon := \max_{\vx, i} \left|f_i(\vx) - \widetilde{\eta}_i(\vx)\right|$.

\begin{theorem}
\label{thm:multiclass}
Assume $\eta(\vx)$ fulfills multi-class Tsybakov condition for constants $C, \lambda>0$ and $t_0 \in (0,1]$. Assume that $\epsilon \leq t_0\min\limits_i \tau_{i,i}$.
For $\Delta = \min\left[1, \min\limits_\vx[\tau_{\ynoise, \ynoise}\eta_{s_\vx}(\vx)+\sum\limits_{j\neq \ynoise}\tau_{j,\ynoise}\eta_j(\vx)]\right]$:
\[ {
\Pr_{(x,y)\sim D} \Big[\ynoise = h^*(\vx), f_{\ynoise}(\vx) < \Delta\Big] \leq C\left[O(\epsilon)\right]^\lambda 
}
\]
\end{theorem}

The proof of Theorem~\ref{thm:multiclass} will be provided in supplementary material.

\subsection{Empirical Validation of the Bound} 
\label{sec:validation}
To better understand the Tsybakov condition assumption and the bound in our theorem, we conduct the following experiment. On the CIFAR10 dataset, we train deep neural networks to approximate relevant functions. We use these functions to estimate the constants $C$ and $\lambda$ in the Tsybakov condition. Using these constants, we calculate the bound in Theorem \ref{thm:multiclass} as a function of $\epsilon$ and check if it is tight.

\begin{figure*}[htb!]
\centering
\begin{tabular}{cc}
\includegraphics[width=.37\textwidth]{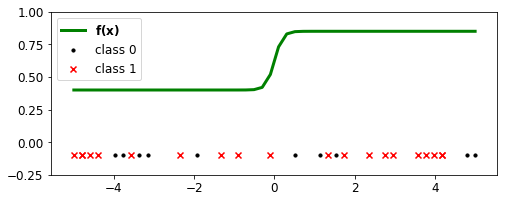} &
\includegraphics[width=.37\textwidth]{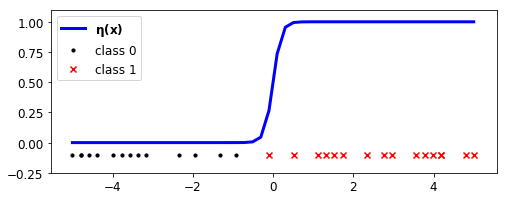} \\
(a) Noisy labels and $f$. & (b) Corrected labels and $\eta$.  \\
\includegraphics[width=.37\textwidth]{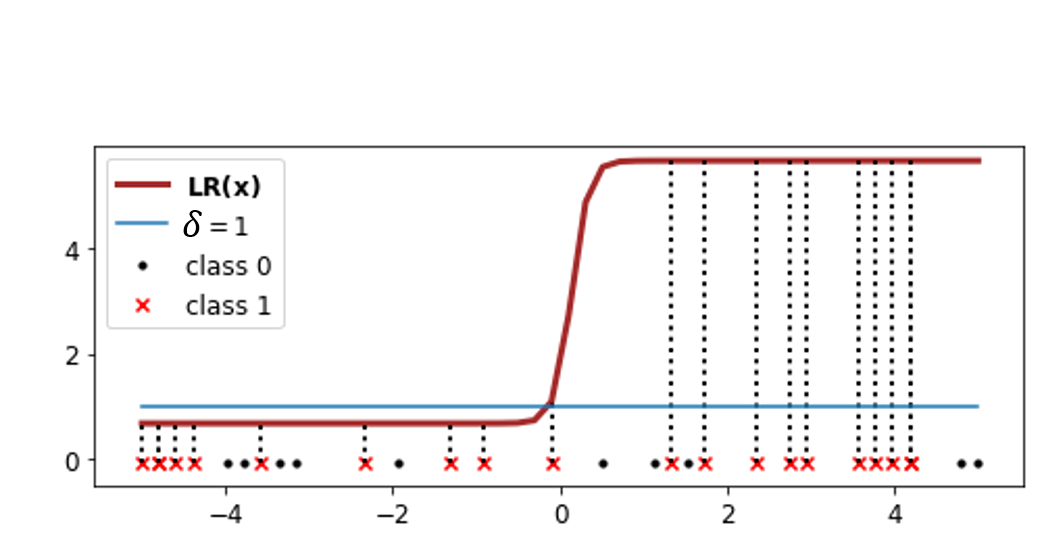} &
\includegraphics[width=.37\textwidth]{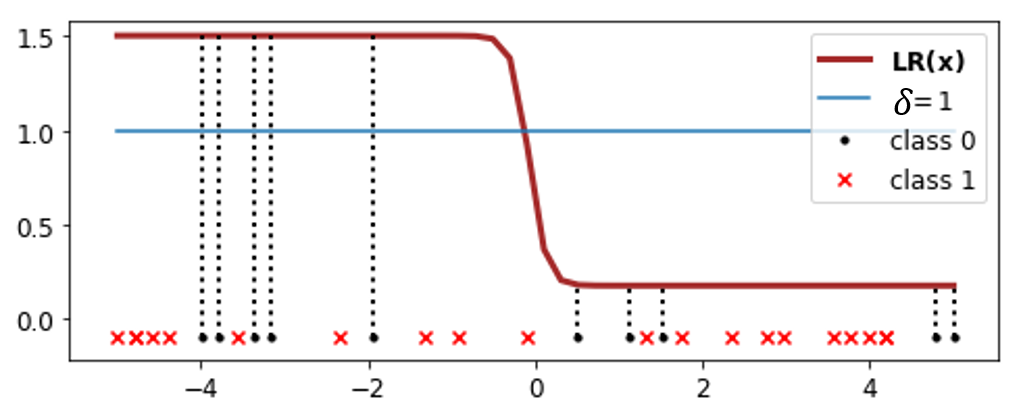} \\
(c) $\LR$ for $\ynoise=1$. & (d) $\LR$ for $\ynoise=0$.
\end{tabular}
\caption{An illustration of the label correction algorithm. $\delta$ is set to 1. (a): a corrupted sample and its corresponding classifier prediction $f$. (b): after correction, the labels are consistent with the true conditional probability, $\eta$. (c): likelihood ratio for $\ynoise=1$. Data with $x<0$ are corrected to $\etanoise_{new}=0$ as $LR(\vx)$ are below $\delta=1$. (d): likelihood ratio for $\ynoise=0$. Data with $x>0$ are corrected to $\etanoise_{new}=1$ as $LR(\vx)$ are below $\delta=1$.}
\label{fig:synth}
\vspace{-.2in}
\end{figure*}
To estimate $C$ and $\lambda$, we approximate the true conditional probability $\eta$ using a deep neural network trained on the original clean-labeled CIFAR10 data. We densely sample $t$ between 0 and 0.9. For each $t$, we empirically evaluate the left hand side (LHS) probability of Equation (\ref{eq:multiclass-tsybakov}) and then use these values to estimate $C$ and $\lambda$ via regression.
In particular, for each $t$ we calculate LHS of Equation (\ref{eq:multiclass-tsybakov}) using the frequency $p_t = \frac{1}{n}\sum_{i=1}^n \mathbf{1}_{\{\eta_{m_\vx}(\vx)-\eta_{s_\vx}(\vx) \leq t\}}(\vx)$, in which $n$ is the number of data. 
If the RHS bound is tight, we can use $\log p_t$ to approximate $\log(Ct^\lambda) $.  $\log(Ct^\lambda) = \log C + \lambda \log t$. As shown in Figure \ref{fig:cifar10_synthetic}(a), we plot all $(\log t, \log(Ct^\lambda))$ pairs as blue dots and estimate $C$ and $\lambda$ via linear regression (red line).
We observe that the samples are quite close to linear. Indeed, we could get ordinary least square (OLS) estimator of constant $C$ and $\lambda$ with high confidence (determinant coefficient $R^2 = 0.99$, p-value $ < 1e-53$). The estimated $C$ and $\lambda$ are $0.23$ and $1.04$ respectively.

Next, we verify our bound in Theorem \ref{thm:multiclass}. Using the estimated $C$ and $\lambda$, we can calculate the bound (RHS of Equation (\ref{eq:multiclass-tsybakov})) as a function of $\epsilon$ (the constant in the big-O is provided in the supplemental material). In Figure \ref{fig:cifar10_synthetic}(b), we plot the bound function in green curve. We compare this bound with the LHS of Equation (\ref{eq:multiclass-tsybakov}) which we can empirically evaluate. In particular, we train a noisy classifier $f$ by training a neural network on noisy labels (symmetric noise level 20\%, see Section \ref{sec:exp} for details). Using $f$, we can count the number of data points which has $f_{\ynoise} \leq \Delta$ and meanwhile $\ynoise$ is equal to $h^\ast(\vx)$ (calculated using $\eta$: the clean-label-trained neural network). This gives us the LHS of Equation (\ref{eq:multiclass-tsybakov}), which is the probability of a label being correct when $f$ has low confidence (blue line in Figure \ref{fig:cifar10_synthetic}(b)). Similarly, we can calculate the probability of a label being correct when $f$ has high confidence (orange line in Fig.~\ref{fig:cifar10_synthetic}(b)). We also carry out the same experiment on a different noise setting (asymmetric noise level 20\%, see Sec.~\ref{sec:exp} for details).

\textbf{Discussion.}
On CIFAR10 dataset, we estimated the constants of Tsybakov condition to be $C=0.23$ and $\lambda=1.04$ with high confidence. This means our bound (Equation (\ref{eq:multiclass-tsybakov}) is almost linear. As observed in Figure \ref{fig:cifar10_synthetic}(b) and (c), the bound is rather small (only up to 0.2 when the approximation error of the classifier, $\epsilon$, is below 0.4). Furthermore, the empirically evaluated chance of $\ynoise$ being correct when $f$ has low confidence (blue lines Figure \ref{fig:cifar10_synthetic}) is almost zero, well below the curve of the bound. In Figure \ref{fig:cifar10_synthetic}(b) The fact that the blue and green line intersects at $\epsilon=0.06$ implies that $\epsilon$ can be as small as 0.06. Similarly, Figure \ref{fig:cifar10_synthetic}(c) implies $\epsilon$ can be as small as 0.12. Finally, we note that the orange lines are well above the blue ones. This means when $f$ has high confidence on $\ynoise$, there is a high chance $\ynoise$ is correct. In other words, by comparing $f_{\ynoise}$ with a properly chosen constant $\Delta$, we can identify most data with corrupted labels.

We also conduct experiments on synthetic data (generated using multivariate normal distribution). In such case, we can calculate $\eta$ and $\etanoise$ exactly. The estimated $C$ and $\lambda$ are $0.6$ and $1.3$ respectively. More details about the synthetic experiments can be found in the supplemental material. 

In conclusion, experiments on synthetic and on CIFAR10 datasets show that the constants in Tsybakov condition are rather small and the bound in our theorem is almost linear to $\epsilon$. We also note the bound is generally small/tight even in deep learning setting. Thresholding $f$'s confidence does detect corrupted labels accurately.

\section{The Algorithm: Likelihood Ratio Test for Label Correction}

Our theoretical insight inspires a new algorithm for label correction. We propose to directly test the confidence level of the noisy classifier to determine whether a label is correct. 
One additional requirement is that if we decide that a label is incorrect, we also need to decide what is the correct label. Therefore, instead of checking the confidence level, we check the likelihood ratio between $f$'s confidence on $\ynoise$ and its confidence on its own label prediction, i.e., $m_\vx$. Specifically, we check the likelihood ratio 
\begin{equation*}
\LR(f,\vx,\ynoise) = f_{\ynoise}(\vx) / f_{m_\vx}(\vx).
\label{eq:lr}
\end{equation*}

We compare this likelihood ratio with a predetermined threshold $\delta$. The value of $\delta$ is given in the next theorem.
 This is essentially a hypothesis testing on the null hypothesis \mbox{$H_0: \ynoise = h^{\ast}(\vx)$}. If $\LR(f,\vx,y) < \delta$, we reject the null hypothesis and flip the label $\ynoise_{new} = m_\vx$. Otherwise, the label remains unchanged, $\ynoise_{new} = \ynoise$. If $\ynoise=m_\vx$ then the likelihood ratio is 1, $\ynoise_{new} = m_\vx = \ynoise$. Detailed algorithm is provided in Procedure \ref{alg:LRT-multi}.
See Figure \ref{fig:synth} for an illustration of the algorithm in a binary classification case. 

\vspace{-.2in}
\begin{table}[h]
\centering
~
\begin{minipage}[t]{0.5\textwidth}
\begin{algorithm}[H]
\caption{\texttt{LRT-Correction}}
\label{alg:LRT-multi}
\begin{algorithmic}[1]
\REQUIRE $(\vx, \ynoise), f(\vx), \delta$.
\ENSURE $\widetilde{y}_{new}$
\STATE $m_\vx:=\argmax_{i} f_i(\vx)$
\STATE $\LR(f,\vx,\ynoise):= f_{\ynoise}(\vx)/f_{m_\vx}(\vx)$
\IF {$\LR(f,\vx,\ynoise) < \delta$}
\STATE $\widetilde{y}_{new} = m_\vx$
\ELSE 
\STATE $\widetilde{y}_{new} = \widetilde{y}$
\ENDIF
\end{algorithmic}
\end{algorithm}
\end{minipage}
~~~ 
\end{table}
\vspace{-.1in}


We will show in the following theorem that the LRT correction algorithm is guaranteed to make proper correction and clean most of the corrupted labels. In particular, we show that in practice if we have a reasonable approximation $\hat{\delta}$ to the theoretically optimal $\delta$, the algorithm flips $\ynoise$ to the correct label (the Bayes optimal prediction, $h^{\ast}(\vx)$) with a good chance. Recall the approximation error of the classifier is $\epsilon := \max_{\vx, i} \left|f_i(\vx) - \widetilde{\eta}_i(\vx)\right|$. 

We consider two cases: (1) the label being flipped $y_{new} = m_{\vx}$; and (2) the label remaining the same $y_{new} = \ynoise$. Each case has its own ideal $\delta$. We bound the probability of obtaining a correct label with $\epsilon$ and $\xi$. Here $\xi$ is the difference between the chosen $\hat{\delta}$ and the ideal $\delta$. We also introduce an additional term, $\Psi$, denoting the probability that the true label is neither $\ynoise$ nor $m_\vx$, formally, $\Psi = \pr_{(\vx, y)\sim D}\left[u_\vx \notin \{m_\vx, \ynoise\}\right]$.

\begin{theorem}
$\forall i,j \in [N_c]$, assume $\eta(\vx)$ fulfills multi-class Tsybakov condition for constants $C>0$, $\lambda>0$, $t_0\in(0,1]$. 

Case 1 (Label flipped by \texttt{LRT-Corr($(\vx,\ynoise)$,$f(\vx)$,$\hat{\delta}$)}): let $\delta_1 =  \min\limits_\vx \left[ \frac{1}{f_{m_\vx}(\vx)}\left(\tau_{\ynoise, \ynoise}\eta_{s_\vx}(\vx)+\sum\limits_{j\neq \ynoise}\tau_{j,\ynoise}\eta_j(\vx)\right)\right]$ and $\xi_1 := |\hat{\delta}-\delta_1|$. Assume $\xi_1\leq \delta_1$ and  $\epsilon \leq \min\left(\frac{t_0 \delta_1^2 \min_i\tau_{ii}-\xi_1^2-\xi_1}{\delta_1^2}, (t_0-\xi_1)\min\limits_i \tau_{ii}\right)$. Then:
 $\Pr_{(x,y)\sim D} \left[ \ynoise_{new} = {h^*(\vx)} , \text{$\ynoise$ is flipped }\right]$ is at least \mbox{\small{$1 - C \left[ O(\max(\epsilon,\xi_1)) \right]^{\lambda} - \Psi$}}.

Case 2 (Label preserved by \texttt{LRT-Corr($(\vx,\ynoise)$,$f(\vx)$,$\hat{\delta}$)}):
let $\delta_2 = \max\limits_\vx \left[\frac{f_{\ynoise}(\vx)}{\tau_{m_\vx, m_\vx}\eta_{s_\vx}(\vx)+\sum\limits_{j\neq m_\vx}\tau_{j,m_\vx}\eta_j(\vx)}\right]$ and $\xi_2 := |\hat{\delta}-\delta_2|$. Assume $\xi_2 \leq \delta_2$ and $\epsilon \leq \min\left(\frac{t_0 \delta_2^2 \min_i\tau_{ii}-\xi_2^2-\xi_2}{\delta_2^2}, (t_0-\xi_2)\min\limits_i \tau_{ii}\right)$.\\ Then: $\Pr_{(x,y)\sim D} \left[ \ynoise_{new} = h^*(\vx), \text{ $\ynoise$ isn't flipped} \right]$ is at least \small{\mbox{$1-  C\left[ O(\max(\epsilon, \xi))\right]^{\lambda}-\Psi$}}. 
\label{thm:main3}
\end{theorem}
\vspace{-1em}

\subsection{Training Deep Nets with LRT-Correction}\label{hack}
\label{sec:dnn-training}

We incorporate the proposed label-correction into the training of deep neural networks.
Similar to other data-re-calibrating methods, our training algorithm continuously trains a deep neural network while correcting the noisy labels.
Procedure \ref{alg:euclid} is the pseudocode of the  training method, called \AdaCorr. It trains a neural network model iteratively. Each iteration includes both label correction and model training steps. In label correction step, the prediction of the current neural network, $f$, is used to run LRT test on all training data, and to correct their labels according to the test result.
Since $f$ is used to approximate the conditional probability $\etanoise$, we use the softmax layer output of the neural network as $f$.
After the labels of all training data are updated, we use them to train the neural network incrementally. 
We continue this iterative procedure until the training converges.

We also have a burn-in stage in which we train the network using the original noisy labels for $m$ epochs. During the burn-in stage, we use the original cross-entropy loss, $\loss_{CE}$. Afterwards, we add an additional retroactive loss, with the intention of stabilizing the network and avoiding overfitting.  

After the burn-in stage, we want to avoid overfitting of the neural network, so that its output better approximates $\etanoise$.
To achieve this goal, we introduce a \emph{retroactive loss} term $\loss_{retro}(f(\vx),\ynoise)$. The idea is to enforce the consistency between $f$ and the prediction of the model at a previous epoch, $f^{r}$. It has been observed that a neural network at earlier training stage tends to learn the true pattern rather than to overfit the noise \citep{Arpit_Memorization_ICML2017}.
Formally, the loss can be written as $\sum_{c=1}^{N_c} f^{r}_c(\vx)\log f_c(\vx)$, in which $N_c$ is the number of possible label classes. The training loss is the sum of the retroactive loss and the cross-entropy loss: 
\begin{align*}
\loss(f(\vx),\ynoise,f^{r}) 
&= \loss_{retro}(f(\vx),f^{r}(\vx)) + \loss_{CE}(f(\vx),\ynoise) \\
&= \sum_{c=1}^{N_c} f^{r}_c(\vx)\log f_c(\vx)+\sum_{c=1}^{N_c} \ynoise_c\log f_c(\vx).
\end{align*}
\vspace{-0.12in}

\begin{algorithm}[htbp!]
\caption{\AdaCorr}\label{alg:euclid}
\begin{algorithmic}[1]
\REQUIRE $S = \{(\vx,\ynoise)\}$, $\delta$, $m$, $T$
\FOR{epoch=1 to $m$} 
\STATE Train neural network with $\loss_{CE}$
\ENDFOR
\STATE $f^{r} = $ current model prediction
\FOR{epoch=$m+1$ to $T$} 
\IF{epoch $\geq m+10$}
\STATE $f=$ current model prediction
\FORALL{$(\vx,\ynoise)\in S$}
\STATE $\ynoise_{new}$= \LRTCorr($(\vx,\ynoise)$,$f$,$\delta$)
\STATE $\ynoise = \ynoise_{new}$
\ENDFOR
\ENDIF
\STATE Train using $\loss_{retro}+\loss_{CE}$, with $f^{r}$ and $\ynoise$
\ENDFOR
\end{algorithmic}
\end{algorithm}
\vspace{-.1in}

In the experiment we evaluate our method on 4 public datasets: CIFAR10, CIFAR100, MNIST and ModelNet40 (see Section \ref{sec:exp} for more details). Based on previous observations \citep{Arpit_Memorization_ICML2017}, on CIFAR10 and CIFAR100 datasets, a neural network takes about 30 epochs to fit the true pattern before overfitting the noise. We use this number as the burn-in stage length $m$. For easier datasets like MNIST and ModelNet40, we set $m$ to be slightly smaller (25). As for $\delta$, setting $\delta$ to be slightly smaller than 1 seems sufficient. Our Theorem \ref{thm:main3} guarantees that the bound is affected almost linearly (as $\lambda\approx 1$ per Section \ref{sec:validation}) to the error of the manually picked $\delta$ from the optimal one. 

\section{Experiments}
\label{sec:exp}
In this section we empirically evaluate our proposed method with several datasets, where noisy labels are injected according to specified noise transition matrices. 

\textbf{Datasets.} We use the following datasets: MNIST (\citealt{mnist}), CIFAR10 (\citealt{cifar100}), CIFAR100 (\citealt{cifar100}) and ModelNet40 (\citealt{modelnet40}). 
MNIST consists of $28 \times 28$ grayscale images with 10 categories. It contains 60,000 images, and we use 45,000 for training, 5,000 for validation and 10,000 for testing. CIFAR10 and CIFAR100 consist of the same 60,000 images whose size is $ 32 \times 32 \times 3$. CIFAR10 has 10 classes while CIFAR100 has 100 fine-grained classes. Similar to MNIST, we split 90\% and 10\% data from the official training set for the training and validation, respectively, and use the official test set for testing. ModelNet40 contains 12,311 CAD models from 40 categories, where 8,859 are used for training, 984 for validation and the remaining 2,468 for testing. We follow the protocol of \cite{PointNet} to convert the CAD models into point clouds by uniformly sampling 1,024 points from the triangular mesh and normalizing them within a unit ball. In all experiments, we use early stopping on validation set to tune hyperparameters and report the performance on test set.

\textbf{Baselines.} We compare the proposed method with the following methods: (1) \textit{Standard}, which trains the network in a standard manner, without any label resistance technique; (2) \textit{Forward Correction} (\citealt{Patrini:2017}), which explicitly estimates the noise transition matrix to correct the training loss; (3) \textit{Decoupling} (\citealt{Malach:2017}), which trains two networks simultaneously and updates the parameters on selected data whose labels are possibly clean; (4) \textit{Coteaching} (\citealt{Han:2018}), which also trains two networks but exchanges their error information for network updating; (5) \textit{MentorNet} (\citealt{jiang:2018}), which learns a curriculum to filter out noisy data; (6) \textit{Forgetting} \citep{Arpit_Memorization_ICML2017}, which uses dropout to help deep models resist label noise. (7) \textit{Abstention} (\citealt{sunil:2019}), which regularizes the network with abstention loss to ensure model robustness under label noise.

\textbf{Experimental setup.} For the classification of MNIST, CIFAR10 and CIFAR100, we use preactive ResNet-34 (\citealt{he2016identity}) as the backbone for all the methods. On ModelNet40, we use PointNet. We train the models for 180 epochs to ensure that all the methods have converged. We utilize RAdam (\citealt{Radam}) for the network optimization, and adopt batch size 128 for all the datasets. We use an initial learning rate of 0.001, which is decayed by 0.5 very 60 epochs. We also update $f^r$ to $f$ once at epoch $m+40$ to reflect better predictive power of network after several epochs. The experimental results are listed in Table~\ref{tab_results}. As is shown, our method overall achieves the best performance across the datasets under different noise settings. 

\textbf{Clothing 1M.} We also evaluate our method on a large scale Clothing 1M dataset \citep{Xiao:2015}, which consists of 1M images with real-world noisy labels. We use pre-trained ResNet-50 and train the model using SGD for 20 epochs. Our method achieves accuracy 71.47\%. It outperforms Standard (68.94\%), Forward Correction (69.84\%) and Backward Correction \cite{Patrini:2017} (69.13\%), where we take the number from the original paper directly. Note that other baselines (Forgetting, Decoupling, MentorNet, Coteaching and Abstention) did not report results on this dataset.

\begin{table}[h]
\vspace{-12pt}
\caption{Performance on Clothing 1M Dataset}
\label{cloth1M}
\begin{center}
\begin{tabular}{cc}
\toprule
Method & Accuracy($\%$) \\
\midrule
Standard & 68.94 \\
Forward  & 69.84 \\
Backward & 69.13 \\
AdaCorr  & 71.74 $\pm$ 0.12\\
\bottomrule
\end{tabular}
\end{center}
\vspace{-12pt}
\end{table}

\textbf{Discussion.} Our method outperform state-of-the-arts over a broad spectrum of noise patterns and levels. This is due to the relatively simple procedure our theoretically guaranteed algorithm. Looking closely, in Figure \ref{fig:convergence}, we draw convergence curves on CIFAR10 with 0.4 uniform noise. On the left, we show the curves of our proposed AdaCorr method. 
The model continues to flip labels to correct ones. Meanwhile, it fits with the corrected labels $y_{new}$ and the test accuracy on clean labels does not drop. This shows that the model and the label correction are improving in a harmonic fashion and do not collapse. On the right, we show the curves of the Standard method. Without label correction, the model overfits with noisy labels and the performance on test data degrades catastrophically.

\begin{figure*}[hbtp]
    \centering
        \includegraphics[scale=0.45]{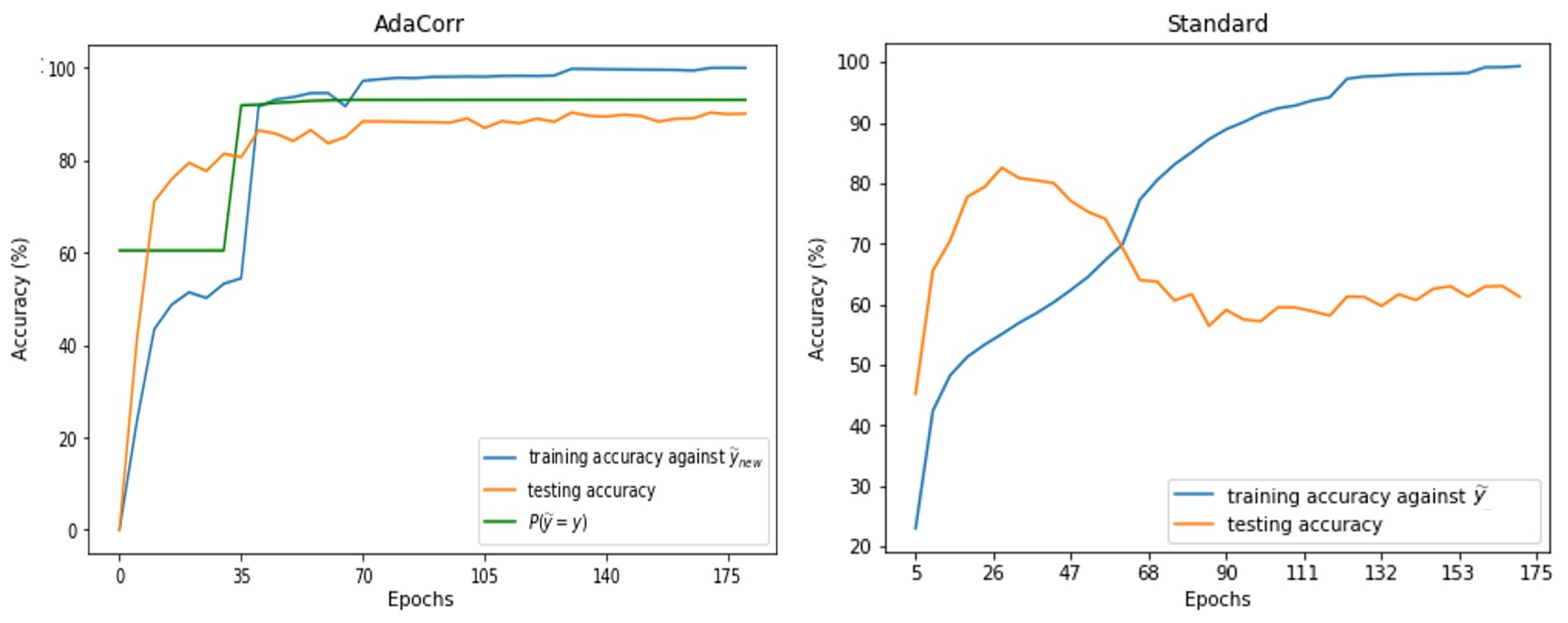}
    \caption{Convergence curves for CIFAR10 with 40\% uniform noise. Left: AdaCorr - training accuracy evaluated against the corrected label ($y_{new}$) (cyan), testing accuracy against clean label (orange), and the proportion of correct label (green). Right: Standard - training accuracy against noisy label ($\widetilde{y}$) and testing accuracy against clean label.}
    
    \label{fig:convergence}
\end{figure*}

\begin{table*}[htbp]
\caption{The classification accuracy of different methods.}
\centering
\begin{tabular}[width=\textwidth]{*6{p{15mm}}|*3{p{15mm}}}
\hline
    \multirow{2}{*}{Data Set} & \multirow{2}{*}{Method} & \multicolumn{4}{c|}{Noise Level of Uniform Flipping} & \multicolumn{3}{c}{Noise Level of Pair Flipping} \\ \cline{3-9} 
    && 0.2 & 0.4 & 0.6 & 0.8 & 0.2 & 0.3 & 0.4\\ \hline
    \multirow{10}{*}{MNIST\vspace{5mm}} 
    & Standard & 99.0 $\pm$ 0.2 & 98.7 $\pm$ 0.4 & 98.1 $\pm$ 0.3 & 91.3 $\pm$ 0.9 & 99.3 $\pm$ 0.1 & 99.2 $\pm$ 0.1 & 98.8 $\pm$ 0.1 \\
    & Forgetting & 99.0 $\pm$ 0.1 & 98.8 $\pm$ 0.1 & 97.7 $\pm$ 0.2 & 62.6 $\pm$ 8.9 & 99.3 $\pm$ 0.1 & 96.5 $\pm$ 2.0 & 89.7 $\pm$ 1.9  \\
    & Forward & 99.1 $\pm$ 0.1 & 98.7 $\pm$ 0.2 & 98.0 $\pm$ 0.4 & 89.6 $\pm$ 4.8 & 99.4 $\pm$ 0.0 & 99.2 $\pm$ 0.2 & 96.5 $\pm$ 4.4\\
    & Decouple & 99.3 $\pm$ 0.1 & 99.0 $\pm$ 0.1 & 98.5 $\pm$ 0.2 & 94.6 $\pm$ 0.2 & 99.4 $\pm$ 0.0 & 99.3 $\pm$ 0.1 & 99.1 $\pm$ 0.2\\
    & MentorNet & 99.2 $\pm$ 0.2 & 98.7 $\pm$ 0.1 & 98.1 $\pm$ 0.1 & 87.5 $\pm$ 5.2 & 98.6 $\pm$ 0.4 & 99.1 $\pm$ 0.1 & 98.9 $\pm$ 0.1\\
    & Coteach & 99.1 $\pm$ 0.2 & 98.7 $\pm$ 0.3 & 98.2 $\pm$ 0.3 & 95.7 $\pm$ 0.7 & 99.1 $\pm$ 0.1 & 99.0 $\pm$ 0.2 & 98.9 $\pm$ 0.2\\
    & Abstention & 94.0 $\pm$ 0.3 & 76.8 $\pm$ 0.3 &  49.6 $\pm$ 0.1 & 21.2 $\pm$ 0.5 & 94.3 $\pm$ 0.3 & 88.5 $\pm$ 0.3 & 81.4 $\pm$ 0.2\\
    \hline
    & AdaCorr & 99.5 $\pm$ 0.0 & 99.4 $\pm$ 0.0 & 99.1 $\pm$ 0.0 & 97.7 $\pm$ 0.2 & 99.5 $\pm$ 0.0 & 99.6 $\pm$ 0.0 & 99.4 $\pm$ 0.0\\
    \hline
    \hline
    
    \multirow{10}{*}{CIFAR10\vspace{5mm}} 
    & Standard & 87.5 $\pm$ 0.2 & 83.1 $\pm$ 0.4 & 76.4 $\pm$ 0.4 & 47.6 $\pm$ 2.0 & 88.8 $\pm$ 0.2 & 88.4 $\pm$ 0.3 & 84.5 $\pm$ 0.3  \\
    & Forgetting & 87.1 $\pm$0.2 & 83.4 $\pm$ 0.2 & 76.5 $\pm$ 0.7 & 33.0 $\pm$ 1.6 & 89.6 $\pm$ 0.1 & 83.7 $\pm$ 0.1 & 86.4 $\pm$ 0.5 \\
    & Forward & 87.4 $\pm$ 0.8 & 83.1 $\pm$ 0.8 & 74.7 $\pm$ 1.7 & 38.3 $\pm$ 3.0 & 89.0 $\pm$ 0.5 & 87.4 $\pm$ 1.1 & 84.7 $\pm$ 0.5\\
    & Decouple & 87.6 $\pm$ 0.4 & 84.2 $\pm$ 0.5 &  77.6 $\pm$ 0.1 & 48.5 $\pm$ 0.9 & 90.6 $\pm$ 0.3 & 89.1 $\pm$ 0.3 & 86.3 $\pm$ 0.5\\
    & MentorNet & 90.3 $\pm$ 0.3 & 83.2 $\pm$ 0.5 & 75.5 $\pm$ 0.7 & 34.1 $\pm$ 2.5 & 90.4 $\pm$ 0.2 & 88.9 $\pm$ 0.1 & 83.3 $\pm$ 1.0\\
    & Coteach & 90.1 $\pm$ 0.4 & 87.3 $\pm$ 0.5 & 80.9 $\pm$ 0.5 & 25.0 $\pm$ 3.6 & 91.8 $\pm$ 0.1 & 89.9 $\pm$ 0.2 & 80.1 $\pm$ 0.7\\
    & Abstention & 85.3 $\pm$ 0.4 & 82.0 $\pm$ 0.7 & 68.8 $\pm$ 0.4 & 33.8 $\pm$ 7.7 & 88.5 $\pm$ 0.0 & 83.1 $\pm$ 0.5 & 77.4 $\pm$ 0.4\\
    \hline
    & AdaCorr & \textbf{91.0 $\pm$ 0.3} & \textbf{88.7 $\pm $ 0.5} & \textbf{81.2 $\pm$ 0.4} & \textbf{49.2 $\pm$ 2.4} 
    & \textbf{92.2 $\pm$ 0.1} & \textbf{91.3 $\pm$ 0.3} & \textbf{89.2 $\pm $ 0.4} \\
    \hline
    \hline
    
    \multirow{10}{*}{CIFAR100\vspace{5mm}}
    & Standard & 58.9 $\pm$ 0.8 & 52.1 $\pm$ 1.0 & 42.1 $\pm$ 0.7 & 20.8 $\pm$ 1.0 & 59.5 $\pm$ 0.4 & 52.9 $\pm$ 0.6 & 44.7 $\pm$ 1.3 \\
    & Forgetting & 59.3 $\pm$ 0.8 & 53.0 $\pm$ 0.2 & 40.9 $\pm$ 0.5 & ~~7.7 $\pm$ 1.1 & 61.4 $\pm$ 0.9 & 54.6 $\pm$ 0.6 & 37.7 $\pm$ 4.6 \\
    & Forward & 58.4 $\pm$ 0.5 & 52.2 $\pm$ 0.3 & 41.1 $\pm$ 0.5 & 20.6 $\pm$ 0.6 & 58.3 $\pm$ 0.7 & 53.2 $\pm$ 0.6 & 44.4 $\pm$ 2.8\\
    & Decouple & 59.0 $\pm$ 0.7 & 52.2 $\pm$ 0.7 & 40.2 $\pm$ 0.4 & 18.5 $\pm$ 0.8 & 60.8 $\pm$ 0.7 &  56.1 $\pm$ 0.7 & 48.4 $\pm$ 1.0\\
    & MentorNet & 63.6 $\pm$ 0.5 & 51.4 $\pm$ 1.4 & 38.7 $\pm$ 0.8 & 17.4 $\pm$ 0.9 & 64.7 $\pm$ 0.2 & 57.4 $\pm$ 0.8 & 47.4 $\pm$ 1.7\\
    & Coteach & 66.1 $\pm$ 0.5 & 60.0 $\pm$ 0.6 & \textbf{48.3 $\pm$ 0.1} & 16.1 $\pm$ 1.1 & 63.4 $\pm$ 0.9 & 57.6 $\pm$ 0.3 & 49.2 $\pm$ 0.3\\
    & Abstention & \textbf{75.1$\pm$ 5.4} & 60.0 $\pm$ 0.8 & 51.1$\pm$ 0.8 & 10.3 $\pm$ 0.5 & 65.4 $\pm$ 0.5 & 56.8 $\pm$ 0.5 & 47.3 $\pm$ 0.3\\
    \hline
    
    & AdaCorr & 67.8 $\pm$ 0.1 & \textbf{60.2 $\pm$ 0.8} & 46.5 $\pm $ 1.2 & \textbf{24.6 $\pm$ 1.1} &
    \textbf{68.3 $\pm$ 0.2} & \textbf{61.1 $\pm$ 0.5} & \textbf{49.8 $\pm$ 0.7}\\
    \hline
    \hline
    
    \multirow{10}{*}{ModelNet40\vspace{5mm}} 
    & Standard & 79.1 $\pm$ 2.6 & 75.3 $\pm$ 3.3& 70.0 $\pm$ 3.0 & 57.9 $\pm$ 2.3 & 84.4 $\pm$ 1.2 & 82.3 $\pm$ 1.3 & 78.9 $\pm$ 0.7 \\
    & Forgetting & 80.1 $\pm$ 1.8 & 73.9 $\pm$ 0.6 & 69.0 $\pm$ 0.7 & 26.2 $\pm$ 4.8 & 83.3 $\pm$ 1.1 & 62.0 $\pm$ 3.0 & 59.5 $\pm$ 2.9 \\
    & Forward & 52.3 $\pm$ 5.1 & 49.4 $\pm$ 6.8 & 43.5 $\pm$ 5.2 & 28.2 $\pm$ 5.5 & 48.1 $\pm$ 6.8 & 48.0 $\pm$ 3.7 & 49.1 $\pm$ 4.4\\
    & Decouple & 82.5 $\pm$ 2.2 & 80.7 $\pm$ 0.7 & 72.9 $\pm$ 1.0 & 55.4 $\pm$ 2.7 & 85.7 $\pm$ 1.4 & 84.3 $\pm$ 1.0 & 80.5 $\pm$ 2.4\\
    & MentorNet & 86.5 $\pm$ 0.5 & 75.4 $\pm$ 1.8 & 70.9 $\pm$ 1.9 & 52.7 $\pm$ 3.1 & 83.7 $\pm$ 1.8 & 81.0 $\pm$ 1.5 & 79.3 $\pm$ 2.1 \\
    & Coteach & 85.6 $\pm$ 0.9 & 84.2 $\pm$ 0.8 & \textbf{81.8 $\pm$ 1.1} & 68.9 $\pm$ 2.8 & 85.7 $\pm$ 0.8 & 79.1 $\pm$ 3.0 & 69.1 $\pm$ 2.4\\
    & Abstention & 78.1 $\pm$ 0.6 & 65.6 $\pm$ 0.5 & 45.6 $\pm$ 1.5 & 23.5 $\pm$ 0.5 & 82.3 $\pm$ 0.5 & 80.4 $\pm$ 0.6 & 65.6 $\pm$ 0.5 \\
    \hline
    & AdaCorr & \textbf{86.9 $\pm$ 0.3} & \textbf{85.1 $\pm$ 0.6} & 78.6 $\pm$ 1.4 & \textbf{72.1 $\pm$ 1.1} &\textbf{87.6 $\pm$ 0.4} & \textbf{84.6 $\pm$ 0.5} & \textbf{83.7 $\pm$ 0.5} \\
    \hline
\end{tabular}
\label{tab_results}
\end{table*}

\section{Conclusion}
We prove theoretical guarantees for data-re-calibrating methods for noisy labels.
Based on the result, we propose a label correction algorithm to combat label noise. Our method can produce models robust to different noise patterns. Experiments on various datasets show that our method outperforms many recently proposed methods.

\newpage
\section*{Acknowledgements}

Mayank Goswami is supported by National Science Foundation grants CRII-1755791 and CCF-1910873. The research of Songzhu Zheng and Chao Chen is partially supported by NSF IIS-1855759, CCF-1855760 and IIS-1909038. The research of Pengxiang Wu and Dimitris Metaxas is partially supported by NSF CCF-1733843. We thank anonymous referees for constructive comments and suggestions.

\bibliographystyle{icml2020}
\bibliography{reference}

\end{document}


\onecolumn 

\icmltitle{Error-Bounded Correction of Noisy Labels\\
--- Supplementary Material --- }



\icmlsetsymbol{equal}{*}

\begin{icmlauthorlist}
\icmlauthor{Songzhu Zheng}{sbu-ams}
\icmlauthor{Pengxiang Wu}{rutgers}
\icmlauthor{Aman Goswami}{bain}
\icmlauthor{Mayank Goswami}{cuny}
\icmlauthor{Dimitris Metaxas}{rutgers}
\icmlauthor{Chao Chen}{sbu-bmi}
\end{icmlauthorlist}

\icmlaffiliation{sbu-ams}{Department of Applied Mathematics and Statistics, Stony Brook University, NY, USA}
\icmlaffiliation{sbu-bmi}{Department of Biomedical Informatics, Stony Brook University, NY, USA}
\icmlaffiliation{bain}{Bain \& Company, Bangalore, India.}
\icmlaffiliation{rutgers}{Department of Computer Science, Rutgers University, NJ, USA}
\icmlaffiliation{cuny}{Department of Computer Science, City University of New York, NY, USA}

\icmlcorrespondingauthor{Songzhu Zheng}{zheng.songzhu@stonybrook.edu}

\printAffiliationsAndNotice{}  

\section{Additional (Synthetic) Experiment for Validation of the Bound}

In Section 2.3 of the submitted manuscript, we used the output of deep neural networks f as an approximation of $\eta$ on the  CIFAR10 dataset. We provided empirical estimates of the constants $C$ and $\lambda$ in the Tsybakov condition for $\eta$, as well as estimates of the probability $\pr[\ynoise = h^{\ast}(\vx), f_{\ynoise}(\vx)<\Delta]$. 

In this section, we provide additional experiments on a \textit{synthetic data set} generated using a mixture-of-Gaussians distribution. In this ideal setting, we know $\eta$, $\tau_{01}$, $\tau_{10}$, $\etanoise$ \textit{exactly}. We can a) use $\etanoise$ as the classifier and b) evaluate the constants in Tsybakov condition for $\eta$ in order to evaluate the upper bound in Theorem 1. 

\textbf{Estimation of Tsybakov condition constants.} We let $\pr(\vx)$ be a mixture of Gaussian distribution in a 10 dimensional feature space, $\vx \sim \frac{1}{2}\mathcal{N}(0, I_{10\times10}) + \frac{1}{2}\mathcal{N}(1, I_{10\times10})$. We sample from the two components with equal probability. If $\vx$ comes from component $\mathcal{N}(0, I_{10\times10})$, it is given label 0. Otherwise, if $\vx$ comes from component $\mathcal{N}(1, I_{10\times10})$, it is given label 1. 
The true conditional distribution is  $\eta(\vx) = \frac{\exp\left\{-\frac{1}{2}\lvert\lvert\vx-1\rvert\rvert^2\right\}}{\exp\left\{-\frac{1}{2}\lvert\lvert\vx\rvert\rvert^2\right\} + \exp\left\{-\frac{1}{2}\lvert\lvert\vx-1\rvert\rvert^2\right\}}$. 

Following the idea of our experiment on CIFAR10 in the manuscript (Section 2.4), we estimate $\pr\left[\lvert \eta(\vx) - \frac{1}{2} \rvert \leq t \right]$ for values of $t$ sampled between 0 and 0.9 using the empirical frequency $p_t=\frac{1}{n}\sum_{i=1}^n \mathbf{1}_{\{\lvert \eta(\vx) - 1/2 \rvert \leq t\}}(\vx)$. Note that if the Tsybakov condition is tight,  $\log(p_t)$ approximates $\log(Ct^\lambda)$. The samples for $\log(t)$ and correspondingly, $\log(Ct^\lambda) \approx \log(p_t)$ are drawn as blue dots in Figure \ref{fig:supp:fig1}(a). The ordinary least square (OLS) linear regression results is drawn as a red line. We found the estimated values of $C$ and $\lambda$ to be 0.58 and 1.27 respectively. The estimation is high is confidence: the determinant coefficient $R^2$ equals $0.904$, and we have a p-value which is less than $10^{-4}$. 

\textbf{Estimation of the error bound, and its tightness.}
We also introduce label noise using predefined transition probability $\tau_{01}$ and $\tau_{10}$. We can estimate $C$ and $\lambda$ as mentioned above, and know $\tau_{01}, \tau_{10}, \eta(x)$, and thus, $\widetilde{\eta}(\vx)$. Therefore we can evaluate the error bound in Theorem 1. We plot the error bound as a function of $\epsilon$ in Figures \ref{fig:supp:fig1}(b) and (c) (drawn green curves). 

Finally, we assume a perfect noisy classifier $f=\etanoise$. In other words, $\epsilon = 0$. We empirically show that when $f(\vx)<\Delta$, the probability of $\ynoise$ being correct (i.e., $\ynoise = h^{\ast}(\vx)$) is zero (blue lines in Figures \ref{fig:supp:fig1}(b) and (c)). 

\begin{figure*}[!htbp]
\begin{tabular}{ccc}
\includegraphics[width=0.33\textwidth]{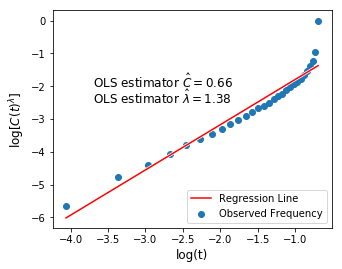}&
\includegraphics[width=0.3\textwidth]{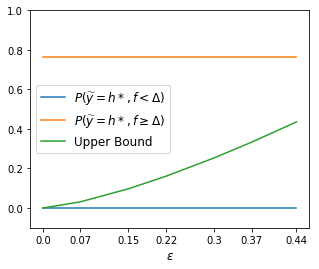} &
\includegraphics[width=0.3\textwidth]{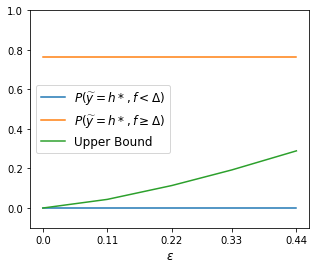} \\
(a) & (b) & (c)\\
\includegraphics[width=0.3\textwidth]{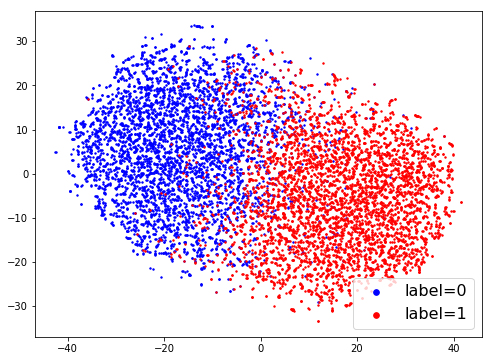} &
\includegraphics[width=0.3\textwidth]{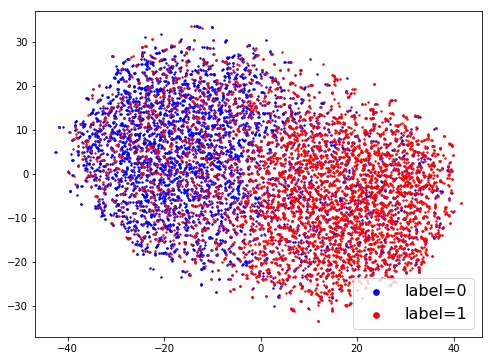} &
\includegraphics[width=0.3\textwidth]{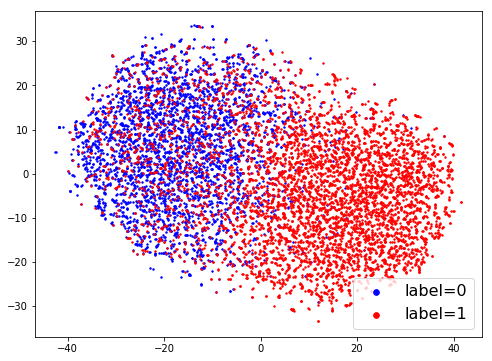} \\
(d) & (e) & (f)
\end{tabular}
\caption{Synthetic experiment using Mixture of Gaussian at noise level 20\%. (a): Check of Tsybakov condition using linear regression, where y-axis is the proportion of data points at distance t from decision boundary. (b): Proportion of labels that are not correct (not consistent with Bayes optimal decision rule) and the proposed upper bound. (c): Same as (b) but labels are corrupted with aysmmetric noise. (d): t-SNE of the clean data. (e): t-SNE of the data with symmetric noise. (f): t-SNE of the data with asymmetric noise.}
\label{fig:supp:fig1}
\end{figure*}

\textbf{Validation of the label-correction algorithm.} To the same synthetic dataset, we also apply our LRT-Correction algorithm and validate the bound in Corollary 1. Since we know $\widetilde{\eta}(\vx)$, $\tau_{01}$ and $\tau_{10}$, we calculate the correction error bound of Corollary 1 in closed form. We draw the bound w.r.t.~the error $\epsilon$ in orange curves in Figure \ref{fig:supp:fig2}. Finally, we run our label correction algorithm using the perfect noisy classifier $f=\etanoise$ and validate that the corrected labels are very close to clean (the success rate is limited by the asymmetry level of the noise pattern). See blue lines in Figure \ref{fig:supp:fig2}.

\begin{figure*}[!htbp]
\label{syn:norm}
\centering
\begin{tabular}{cc}
\includegraphics[width=0.3\textwidth]{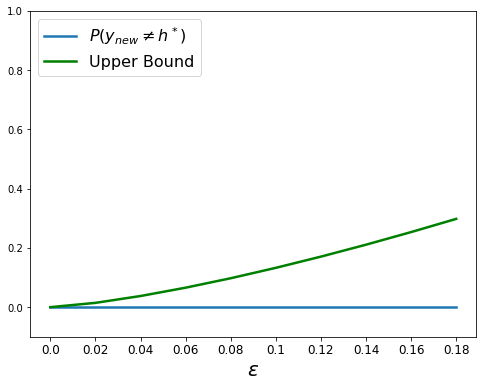} &
\includegraphics[width=0.3\textwidth]{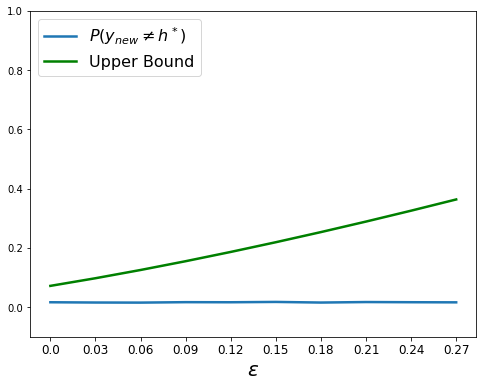}\\ 
(a) & (b)\\
\includegraphics[width=0.3\textwidth]{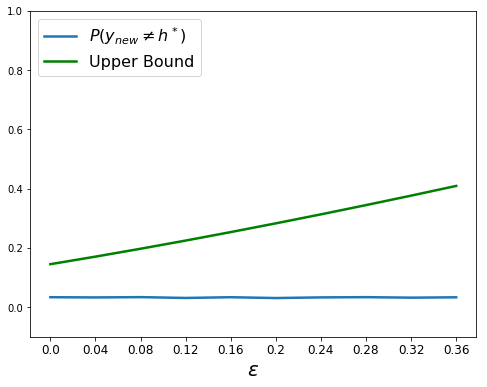} &
\includegraphics[width=0.3\textwidth]{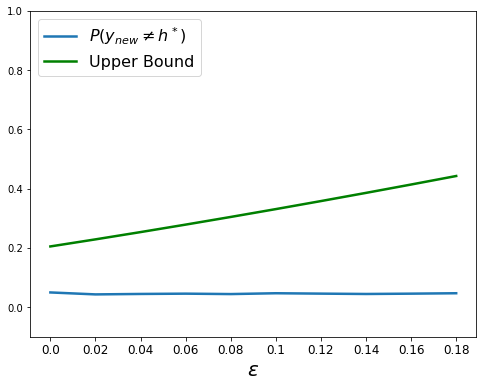}\\ 
(c) & (d)\\
\end{tabular}
\caption{Performance of LRT algorithm given $\widetilde{\eta}(\vx)$ v.s the proposed upper bound. 
(a): Symmetric noise ($\tau_{10}=\tau_{01}=0.3$). (b): Asymmetric noise ($\tau_{10}=0.2, \tau_{01}=0.3$). (c): Asymmetric noise ($\tau_{10}=0.1, \tau_{01}=0.3$). (d): Asymmetric noise ($\tau_{10}=0.3, \tau_{01}=0$)}
\label{fig:supp:fig2}
\end{figure*}

\begin{figure*}
\label{syn:lrtcorrect}
\centering
\begin{tabular}{ccc}
     \includegraphics[width=0.3\textwidth]{figs/synthetic_experiment/tsne_clean.png} &
     \includegraphics[width=0.3\textwidth]{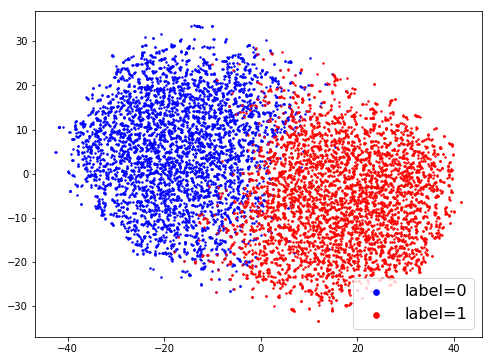} &
     \includegraphics[width=0.3\textwidth]{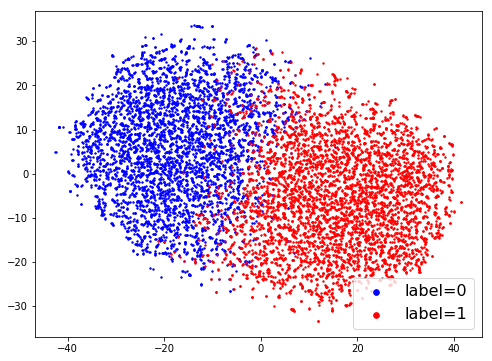}\\
     (a) & (b) & (c)
\end{tabular}
\caption{Label Correction Result Using LRT-Correct. (a): Clean data as it in Fig 1d.
(b): Labels after correction for data in Fig 1e. (c): Labels after correction for data in Fig 1f.}
\end{figure*}

\newpage

\section{Proof of Theorem 2}



Define $m_\vx := \argmax\limits_i f_i(\vx)$, $u_\vx := \argmax\limits_i \eta_i(\vx)$ and $s_\vx := \argmax\limits_{i\neq u_\vx} \eta_i(\vx)$. Let $[Nc] := \{1,2,\cdots, N_c\}$. Finally, define $\epsilon_{i}(\vx) := \left|f_i(\vx) - \widetilde{\eta}_i(\vx)\right|$ and $\epsilon := \max\limits_{\vx, i} \epsilon_i(\vx)$.

For multi-class scenario, we know $\forall i\in [N_c]$, $\widetilde{\eta}_i(\vx) = \sum\limits_{j\in [N_c]}\tau_{ji}\eta_j(\vx)$. We also restate the multi-class Tsybakov condition here:

\begin{assumption}[Multi-class Tsybakov Condition] 
$\exists C,\lambda >0$ and $t_0 \in (0,1]$ such that for all $t \leq t_0$,  
\[
    \pr\left[\left|\eta_{u_\vx}(\vx) - \eta_{s_\vx}(\vx)\right| \leq t\right] \leq Ct^\lambda
\]
\label{assume:multiclass}
\end{assumption}

\setcounter{theorem}{1}

\setcounter{theorem}{1}
\begin{theorem}
\label{thm:main2}
Assume $\eta(\vx)$ fulfills multi-class Tsybakov condition for constant $C, \lambda>0$ and $t_0 \in (0,1]$. Assume that $\epsilon \leq t_0\min\limits_i \tau_{i,i}$.
For $\Delta = \min\left[1, \min\limits_\vx[\tau_{\ynoise, \ynoise}\eta_{s_\vx}(\vx)+\sum\limits_{j\neq \ynoise}\tau_{j,\ynoise}\eta_j(\vx)]\right]$:
\[ {
\Pr_{(x,y)\sim D} \Big[\ynoise = h^*(\vx), f_{\ynoise}(\vx) < \Delta\Big] \leq C\left[O(\epsilon)\right]^\lambda 
}
\]
\end{theorem}
\vspace{-2em}
\begin{proof}

\begin{align}
\label{thm2:part1}
&\pr\left[\ynoise = h^*(\vx),  f_{\ynoise}(\vx) < \Delta \right] 
= \pr\left[\eta_{\ynoise}(\vx) \geq \eta_{s_\vx}(\vx), f_{\ynoise}(\vx) < \Delta \right] \nonumber \\
&\leq \pr\left[\eta_{\ynoise}(\vx) \geq \eta_{s_\vx}(\vx), \widetilde{\eta}_{\ynoise}(\vx) < \Delta + \epsilon_{\ynoise} \right] \nonumber \\
&\leq \pr\left[\eta_{\ynoise}(\vx) \geq \eta_{s_\vx}(\vx), \widetilde{\eta}_{\ynoise}(\vx) < \Delta + \epsilon \right] \nonumber \\
&= \pr\left[\eta_{\ynoise}(\vx) \geq \eta_{s_\vx}(\vx), \sum\limits_{j\in [N_c]}\tau_{j,\ynoise}\eta_{\ynoise}(\vx) < \Delta + \epsilon \right] \nonumber \\
&= \pr\left[\eta_{\ynoise}(\vx) \geq \eta_{s_\vx}(\vx), \eta_{\ynoise}(\vx) < \frac{\Delta - \sum\limits_{j\neq \ynoise}\tau_{j,\ynoise}\eta_{j}(\vx)+\epsilon}{\tau_{\ynoise, \ynoise}}\right] \nonumber \\
&= \pr\left[\eta_{s_\vx}(\vx) \leq \eta_{\ynoise}(\vx) < \frac{\Delta-\sum\limits_{j\neq \ynoise}\tau_{j,\ynoise}\eta_j(\vx)}{\tau_{\ynoise, \ynoise}}+\frac{\epsilon}{\tau_{\ynoise, \ynoise}}\right]
\end{align}


Remember that $\Delta = \min\left[1, \min\limits_\vx [\tau_{\ynoise, \ynoise}\eta_{s_\vx}(\vx)+\sum\limits_{j\neq \ynoise}\tau_{j,\ynoise}\eta_j(\vx)] \right] \leq \tau_{\ynoise, \ynoise}\eta_{s_\vx}(\vx)+\sum\limits_{j\neq \ynoise}\tau_{j,\ynoise}\eta_j(\vx)$. Then if we substitute $\Delta$ in (\ref{thm2:part1}) with $\tau_{\ynoise, \ynoise}\eta_{s_\vx}(\vx)+\sum\limits_{j\neq \ynoise}\tau_{j,\ynoise}\eta_j(\vx)$, continuing the derivation of (\ref{thm2:part1}), we will end up with:

\begin{align*}
&\pr\left[\ynoise = h^*(\vx), f_{\ynoise}(\vx) < \Delta \right] \\
&\leq \pr\left[\eta_{s_\vx}(\vx) \leq \eta_{\ynoise}(\vx)  < \frac{\Delta - \sum\limits_{j\neq \ynoise}\tau_{j,\ynoise}\eta_j(\vx) }{\tau_{\ynoise, \ynoise}} + \frac{\epsilon}{\tau_{\ynoise, \ynoise}}\right]\\
&\leq \pr\left[\eta_{s_\vx}(\vx) \leq \eta_{\ynoise}(\vx) < \eta_{s_\vx}(\vx) + \frac{\epsilon}{\tau_{\ynoise, \ynoise}}\right]
\leq C\left( \frac{\epsilon}{\tau_{\ynoise, \ynoise}} \right)^\lambda
\end{align*}

Notice that Tsybakov condition holds here because $\epsilon \leq t_0 \min\limits_i \tau_{i,i}$, which implies that $\frac{\epsilon}{\tau_{\ynoise, \ynoise}} \leq t_0$. This complete the proof for this case. 

\end{proof}


\section{Proof of Theorem~\ref{theorem3}}

\begin{lemma}
\label{lemma1}
(Algorithm Multiclass-Theorem Guarantee). Assume $\eta(\vx)$ fulfills multi-class Tsybakov condition for constant $C>0$, $\lambda>0$ and $t_0 \in (0,1]$. Assume that $\epsilon \leq t_0 \min\limits_i \tau_{ii}$. Let $\widetilde{y}_{new}$ denote the output of the \LRTCorr~with $\vx$, $\ynoise_\vx$, $f$, and the given $\delta$, then:
\begin{enumerate}
\item{Sensitivity Optimized Critical Value.} Let $\delta =  \min\limits_\vx \left[ \frac{\tau_{\ynoise, \ynoise}\eta_{s_\vx}(\vx)+\sum\limits_{j\neq \ynoise}\tau_{j,\ynoise}\eta_j(\vx)}{f_{m_\vx}(\vx)}\right]$ then :
\[
\Pr_{(x,y)\sim D} \left[ \ynoise_{new} \neq {h^*(\vx)} , \ynoise \text{ is rejected}\right] \leq C\left[O(\epsilon)\right]^\lambda + \pr\limits_{(\vx, y)\sim D}\left[u_\vx \neq m_\vx, u_\vx \neq \ynoise\right]
\]
\item{Specificity Optimized Critical Value.} Let $\delta = \max\limits_\vx \left[\frac{f_{\ynoise}(\vx)}{\tau_{m_\vx, m_\vx}\eta_{s_\vx}(\vx)+\sum\limits_{j\neq m_\vx}\tau_{j,m_\vx}\eta_j(\vx)}\right]$ then :
\[
\Pr_{(x,y)\sim D} \left[ \ynoise_{new} \neq h^*(\vx) , \ynoise \text{ is accepted} \right] \leq C\left[O(\epsilon)\right]^\lambda + \pr\limits_{(\vx, y)\sim D}\left[u_\vx \neq m_\vx, u_\vx \neq \ynoise\right]
\]
\end{enumerate}
\end{lemma}

\begin{proof}
First look at cases where $\ynoise$ is rejected. 
\begin{align}
\label{lemma:part1}
&\pr\left[\ynoise_{new} \neq h^*(\vx) , \ynoise \text{ is rejected}\right]  \nonumber \\
&= \pr\left[\ynoise_{new} \neq h^*(\vx), \frac{f_{\ynoise}(\vx)}{f_{m_\vx}(\vx)} < \delta\right] \nonumber \\
&= \pr\left[\ynoise_{new} = m_\vx \neq h^*(\vx) = \ynoise, \frac{f_{\ynoise}(\vx)}{f_{m_\vx}(\vx)} < \delta\right]+
\pr\left[\ynoise_{new} = m_\vx \neq h^*(\vx) = u_\vx, u_\vx \neq \ynoise, \frac{f_{\ynoise}(\vx)}{f_{m_\vx}(\vx)} < \delta\right] \nonumber\\
&\leq \pr\left[h^*(\vx) = \ynoise, \frac{f_{\ynoise}(\vx)}{f_{m_\vx}(\vx)} < \delta\right]+
\pr\left[\ynoise_{new} = m_\vx \neq h^*(\vx) = u_\vx, u_\vx \neq \ynoise, \frac{f_{\ynoise}(\vx)}{f_{m_\vx}(\vx)} < \delta\right]
\end{align}

For the first term in (\ref{lemma:part1}), we have:
\begin{align}
\label{lemma1:2.1.1}
&\pr\left[h^*(\vx) = \ynoise, \frac{f_{\ynoise}(\vx)}{f_{m_\vx}(\vx)} < \delta\right] 
= \pr\left[h^*(\vx) = \ynoise, {f_{\ynoise}(\vx)} < \delta{f_{m_\vx}(\vx)}\right] \nonumber \\
&\leq \pr\left[\eta_{\ynoise}(\vx) \geq \eta_{s_\vx}(\vx), \widetilde{\eta}_{\ynoise}(\vx) - \epsilon < \delta f_{m_\vx}(\vx) \right]  \nonumber \\
&\leq \pr\left[\eta_{s_\vx}(\vx) \leq \eta_{\ynoise}(\vx)  < \frac{\delta f_{m_\vx}(\vx) - \sum\limits_{j\neq \ynoise}\tau_{j,\ynoise}\eta_j(\vx)}{\tau_{\ynoise, \ynoise}} + \frac{\epsilon}{\tau_{\ynoise, \ynoise}} \right]
\end{align}


We substitute $\delta$ in (\ref{lemma1:2.1.1}) with $\frac{\tau_{\ynoise, \ynoise}\eta_{s_\vx}(\vx)+\sum\limits_{j\neq \ynoise}\tau_{j,\ynoise}\eta_j(\vx)}{f_{m_\vx}(\vx)}$ and continue the calculation:

\begin{align}
\label{lemma1:2.1}
&\pr\left[h^*(\vx) = \ynoise, \frac{f_{\ynoise}(\vx)}{f_{m_\vx}(\vx)} < \delta\right] \nonumber \\
&\leq \pr\left[\eta_{s_\vx}(\vx) \leq \eta_{\ynoise}(\vx)  < \frac{\delta f_{m_\vx}(\vx) - \sum\limits_{j\neq \ynoise}\tau_{j,\ynoise}\eta_j(\vx)}{\tau_{\ynoise, \ynoise}} + \frac{\epsilon}{\tau_{\ynoise, \ynoise}} \right] \nonumber \\
&\leq \pr\left[\eta_{s_\vx}(\vx) \leq \eta_{\ynoise}(\vx) \leq \eta_{s_\vx}(\vx) + \frac{\epsilon}{\tau_{\ynoise, \ynoise}}\right] \nonumber \\
&\leq C\left(\frac{\epsilon}{\tau_{\ynoise, \ynoise}}\right)^\lambda
\end{align}

In (\ref{lemma1:2.1}), the Tsybakov condition holds here because $\epsilon \leq t_0 \min\limits_{i} \tau_{ii} $, which implies $\frac{\epsilon}{\tau_{\ynoise, \ynoise}} \leq t_0$.

For the second term in (\ref{lemma:part1}), we have:
\begin{align}
\label{lemma1:2.2}
&\pr\left[\ynoise_{new} = m_\vx \neq h^*(\vx) = u_\vx, u_\vx \neq \ynoise, \frac{f_{\ynoise}(\vx)}{f_{m_\vx}(\vx)} < \delta\right] 
\leq \pr\left[u_\vx \neq m_\vx, u_\vx \neq \ynoise\right]
\end{align}

for which our algorithm currently doesn't have a good way to deal with and we will leave it as future research problem.

Finally, summarize every piece and we finished the proof for cases where $\ynoise$ is rejected:
\begin{align*}
&\pr\left[\widetilde{y}_{new} \neq h^*(\vx) , \ynoise \text{ is rejected}\right] \leq (\ref{lemma:part1}) \\
&\leq (\ref{lemma1:2.1}) + (\ref{lemma1:2.2}) \\
&\leq C\left[\frac{\epsilon}{\tau_{u_\vx, u_\vx}}\right]^\lambda + \pr\left[u_\vx \neq m_\vx, u_\vx \neq \ynoise\right]\\
&=C\left[O(\epsilon)\right]^\lambda + \pr\left[u_\vx \neq m_\vx, u_\vx \neq \ynoise\right]
\end{align*}

For cases where $\ynoise$ is accepted:
\begin{align}
\label{lemma1:accepte}
&\pr\left[\ynoise_{new} \neq h^*(\vx) , \ynoise \text{ is accepted} \right] 
=\pr\left[\ynoise_{new} \neq h^*(\vx) , \frac{f_{\ynoise}(\vx)}{f_{m_\vx}(\vx)} \geq \delta\right] \nonumber\\
&= \pr\left[\ynoise_{new} = \ynoise \neq h^*(\vx) = m_\vx, \frac{f_{\ynoise}(\vx)}{f_{m_\vx}(\vx)}\geq \delta\right] + 
\pr\left[\ynoise_{new} = \ynoise \neq h^*(\vx), m_\vx \neq h^*(\vx), \frac{f_{\ynoise}(\vx)}{f_{m_\vx}(\vx)}\geq \delta\right] \nonumber \\
&= \pr\left[\eta_{m_\vx}(\vx) \geq \eta_{s_\vx}(\vx), f_{m_\vx}(\vx) \leq f_{\ynoise}(\vx)/\delta\right] + 
\pr\left[u_\vx \neq m_\vx, u_\vx \neq \ynoise\right]
\end{align}

For the first term in (\ref{lemma1:accepte}), we have:
\begin{align}
\label{lemma1:accepte:1}
&\pr\left[\eta_{m_\vx}(\vx) \geq \eta_{s_\vx}(\vx), f_{m_\vx}(\vx) \leq f_{\ynoise}(\vx)/\delta\right] \leq 
\pr\left[\eta_{m_\vx}(\vx) \geq \eta_{s_\vx}(\vx), \widetilde{\eta}_{m_\vx}(\vx) - \epsilon \leq f_{\ynoise}(\vx)/\delta\right] \nonumber\\
&= \pr\left[\eta_{s_\vx}(\vx) \leq \eta_{m_\vx}(\vx) \leq \frac{f_{\ynoise}(\vx)/\delta -\sum\limits_{j\neq m_\vx}\tau_{j,m_\vx}\eta_j(\vx)}{\tau_{m_\vx, m_\vx}} + \frac{\epsilon}{\tau_{m_\vx, m_\vx}}\right]
\end{align}

Firstly, observe that if $\delta > 1$, then $\pr\left[\ynoise_{new} = \ynoise \neq h^*(\vx), \frac{f_{\ynoise}(\vx)}{f_{m_\vx}(\vx)} \geq \delta\right] = 0$ due to the definition of $m_\vx$. 

Then notice that $\delta = \max\limits_\vx \frac{f_{\ynoise}(\vx)}{\tau_{m_\vx, m_\vx}\eta_{s_\vx}(\vx)+\sum\limits_{j\neq m_\vx}\tau_{j,m_\vx}\eta_j(\vx)} \geq \frac{f_{\ynoise}(\vx)}{\tau_{m_\vx, m_\vx}\eta_{s_\vx}(\vx)+\sum\limits_{j\neq m_\vx}\tau_{j,m_\vx}\eta_j(\vx)}$. If we substitute $\delta$ in (\ref{lemma1:accepte:1}) with $\frac{f_{\ynoise}(\vx)}{\tau_{m_\vx, m_\vx}\eta_{s_\vx}(\vx)+\sum\limits_{j\neq m_\vx}\tau_{j,m_\vx}\eta_j(\vx)}$ and continuing the calculation, we will have:
\begin{align}
\label{lemma1:accepte:1.1}
&\pr\left[\eta_{m_\vx}(\vx) \geq \eta_{s_\vx}(\vx), f_{m_\vx}(\vx) \leq f_{\ynoise}(\vx)/\delta\right]  \nonumber\\
&\leq \pr\left[\eta_{s_\vx}(\vx) \leq \eta_{m_\vx}(\vx) \leq \frac{f_{\ynoise}(\vx)/\delta -\sum\limits_{j\neq m_\vx}\tau_{j,m_\vx}\eta_j(\vx)}{\tau_{m_\vx, m_\vx}} + \frac{\epsilon}{\tau_{m_\vx, m_\vx}}\right] \nonumber\\
&\leq \pr\left[\eta_{s_\vx}(\vx) \leq \eta_{m_\vx}(\vx) \leq \eta_{s_\vx}(\vx) + \frac{\epsilon}{\tau_{m_\vx, m_\vx}}\right] \nonumber\\
&\leq C\left[\frac{\epsilon}{\tau_{m_\vx, m_\vx}}\right]^\lambda
\end{align}

For the second term in (\ref{lemma1:accepte}), our algorithm cannot deal with it properly. We will leave it as the future research problem.

Now we summarize all pieces and we get:

\begin{align*}
&\pr\left[\ynoise_{new}\neq h^*(\vx), \ynoise \text{ is accepted}\right] 
= (\ref{lemma1:accepte}) \\
&\leq (\ref{lemma1:accepte:1.1}) + \pr\left[u_\vx \neq m_\vx, u_\vx \neq \ynoise\right] \\
&\leq C\left[\frac{\epsilon}{\tau_{u_\vx, u_\vx}}\right]^\lambda + \pr\left[u_\vx \neq m_\vx, u_\vx \neq \ynoise\right]\\
\end{align*}
which compete the proof for cases that are accepted.
\end{proof}

We give following several facts based on our theorem:
\begin{enumerate}
    \item For binary case, if we set $\delta = \frac{1-|\tau_{10}-\tau_{01}|}{1+|\tau_{10}-\tau_{01}|}$ and further assume $\epsilon \leq t_0(1-\tau_{10}-\tau_{01})-\frac{|\tau_{10}-\tau_{01}|}{2}$, we have:
    \[
        \pr_{(\vx,y)\sim D}\left[\ynoise_{new} \neq h^*(\vx) \right] \leq C\left[\left|\frac{\tau_{10}-\tau_{01}}{2(1-\tau_{10}-\tau_{01})}\right| + \frac{\epsilon}{1-\tau_{10}-\tau_{01}}\right]^\lambda
    \]
    \begin{proof}
    For binary case,  we have:
    \begin{align}
    \label{lemma1:binary}
    &\pr_{(\vx,y)\sim D}\left[\ynoise_{new} \neq h^*(\vx) \right] 
    = \pr_{(\vx,y)\sim D}\left[\ynoise_{new} \neq h^*(\vx), \ynoise \text{ is rejected}\right] + \pr_{(\vx,y)\sim D}\left[\ynoise_{new} \neq h^*(\vx), \ynoise \text{ is accepted}\right] \nonumber\\
    &= \pr\left[\eta_{\ynoise}(\vx) > \frac{1}{2}, \frac{f_{\ynoise}(\vx)}{f_{m_\vx}(\vx)} < \delta \right] + \pr\left[\eta_{\ynoise}(\vx) \leq \frac{1}{2}, \frac{f_{\ynoise}(\vx)}{f_{m_\vx}(\vx)} \geq \delta \right] \nonumber\\
    &\leq \pr\left[\eta_{\ynoise}(\vx) > \frac{1}{2}, \frac{f_{\ynoise}(\vx)}{1-f_{\ynoise}(\vx)} < \delta \right] + \pr\left[\eta_{\ynoise}(\vx) \leq \frac{1}{2}, \frac{f_{\ynoise}(\vx)}{1-f_{\ynoise}(\vx)} \geq \delta \right] \nonumber\\
    &\leq \pr\left[\eta_{\ynoise}(\vx) > \frac{1}{2}, \widetilde{\eta}_{\ynoise}(\vx) < \frac{\delta}{1+\delta} + \epsilon \right] + \pr\left[\eta_{\ynoise}(\vx) \leq \frac{1}{2}, \widetilde{\eta}_{\ynoise}(\vx) \geq \frac{\delta}{1+\delta} - \epsilon \right] \nonumber\\
    &= \pr\left[\frac{1}{2} < \eta_{\ynoise}(\vx) < \frac{\frac{\delta}{1+\delta}-\tau_{1-\ynoise, \ynoise}}{1-\tau_{10}-\tau_{01}}  + \frac{\epsilon}{1-\tau_{10}-\tau_{01}} \right] + \pr\left[ \frac{\frac{\delta}{1+\delta}-\tau_{1-\ynoise, \ynoise}}{1-\tau_{10}-\tau_{01}}  - \frac{\epsilon}{1-\tau_{10}-\tau_{01}}  \leq \eta_{\ynoise}(\vx) \leq \frac{1}{2} \right]
    \end{align}
    
    Observe that $\delta = \frac{1-|\tau_{10}-\tau_{01}|}{1+|\tau_{10}-\tau_{01}|} \leq \frac{1-\tau_{\ynoise, 1-\ynoise}+\tau_{1-\ynoise, \ynoise}}{1+\tau_{\ynoise, 1-\ynoise}-\tau_{1-\ynoise, \ynoise}}$. We also have $\frac{\delta}{1+\delta} = \frac{1-|\tau_{10}-\tau_{01}|}{2} \leq \frac{1}{2}$. Now we substitute $\delta = \frac{1 - \tau_{\ynoise, 1-\ynoise} + \tau_{1-\ynoise, \ynoise}}{1 + \tau_{\ynoise, 1-\ynoise} - \tau_{1-\ynoise, \ynoise}}$ in the first term of (\ref{lemma1:binary}) and substitute $\frac{\delta}{1+\delta}$ with $\frac{1}{2}$ in the second term of (\ref{lemma1:binary}), by algebra we know that :
    \begin{align*}
    &\pr_{(\vx,y)\sim D}\left[\ynoise_{new} \neq h^*(\vx) \right] 
    = \pr_{(\vx,y)\sim D}\left[\ynoise_{new} \neq h^*(\vx), \ynoise \text{ is rejected}\right] + \pr_{(\vx,y)\sim D}\left[\ynoise_{new} \neq h^*(\vx), \ynoise \text{ is accepted}\right]\\
    &\leq \pr\left[\frac{1}{2} < \eta_{\ynoise}(\vx) < \frac{1}{2} + \frac{\epsilon}{1-\tau_{10}-\tau_{01}}\right] + \pr\left[\frac{1/2-\max(\tau_{10},\tau_{01})}{1-\tau_{10}-\tau_{01}} - \frac{\epsilon}{1-\tau_{10}-\tau_{01}} \leq \eta_{\ynoise}(\vx) \leq \frac{1}{2}\right]\\
    & \leq C\left[\left|\frac{\tau_{10}-\tau_{01}}{2(1-\tau_{10}-\tau_{01})}\right| + \frac{\epsilon}{1-\tau_{10}-\tau_{01}}\right]^\lambda
    \end{align*}
    Tsybakov assumption holds because $\frac{\epsilon}{1-\tau_{10}-\tau_{01}} + \frac{|\tau_{10}-\tau_{01}|}{2(1-\tau_{10}-\tau_{01})} \leq \frac{t_0 (1-\tau_{10}-\tau_{01})-\frac{|\tau_{10}-\tau_{01}|}{2}}{1-\tau_{10}-\tau_{01}} + \frac{|\tau_{10}-\tau_{01}|}{2(1-\tau_{10}-\tau_{01})} \leq t_0$. 
    \end{proof}
    
    \item For symmetric noise $\tau_{ij}=\tau_{ji} = \tau, \forall i,j \in [N_c]$ and further assume (besides the assumption we made in Lemma~\ref{lemma1}) $\epsilon \leq \frac{1}{2}\min\limits_{\vx}\left[\widetilde{\eta}_{u_\vx}(\vx) - \widetilde{\eta}_{s_\vx}(\vx)\right]$,  we have:
    \begin{enumerate}
    \item{Sensitivity Optimized Critical Value.} Let $\delta =  \min\limits_\vx \left[ \frac{\tau_{\ynoise, \ynoise}\eta_{s_\vx}(\vx)+\sum\limits_{j\neq \ynoise}\tau_{j,\ynoise}\eta_j(\vx)}{f_{m_\vx}(\vx)}\right]$ then :
    \[
    \Pr_{(x,y)\sim D} \left[ \ynoise_{new} \neq {h^*(\vx)} , \ynoise \text{ is rejected}\right] \leq C\left[O(\epsilon)\right]^\lambda 
    \]
    \item{Specificity Optimized Critical Value.} Let $\delta = \max\limits_{\vx} \left[ \frac{f_{\ynoise}(\vx)}{(\tau_{m_\vx, m_\vx} - \tau_{\ynoise, m_\vx})\eta_{s_\vx}(\vx) + \tau_{\ynoise, m_\vx}}\right]$ then :
    \[
    \Pr_{(x,y)\sim D} \left[ \ynoise_{new} \neq h^*(\vx) , \ynoise \text{ is accepted} \right] \leq C\left[O(\epsilon)\right]^\lambda 
    \]
    \end{enumerate}
    \begin{proof}
    Observe that under symmetric noise scenario, $\forall i \in [N_c]$, $\eta_{u_\vx}(\vx) \geq \eta_{i}(\vx)$ will implies that $\widetilde{\eta}_{u_\vx}(\vx) \geq \widetilde{\eta}_i(\vx)$, i.e. $h^*(\vx) = \widetilde{h}^*(\vx)$. To show this:
    \begin{align*}
    \eta_{u_\vx}(\vx) &\geq \eta_i(\vx) \\\Longleftrightarrow 
    [1-N_c\tau] \eta_{u_\vx}(\vx) &\geq [1-N_c \tau] \eta_{u_i}(\vx) \\ \Longleftrightarrow [1-(N_c-1)\tau] \eta_{u_\vx}(\vx) - \tau \eta_{u_\vx}(\vx) & \geq [1-(N_c-1)\tau] \eta_{i}(\vx) - \tau \eta_i(\vx)\\
    \Longleftrightarrow [1-(N_c-1)\tau] \eta_{u_\vx}(\vx) + \tau \eta_i(\vx)  & \geq [1-(N_c-1)\tau] \eta_{i}(\vx) + \tau \eta_{u_\vx}(\vx) \\
    \Longleftrightarrow [1-(N_c-1)\tau] \eta_{u_\vx}(\vx) + \tau \eta_i(\vx) + \tau \sum\limits_{j \neq u_\vx, j\neq i} \eta_j(\vx) & \geq [1-(N_c-1)\tau] \eta_{i}(\vx) + \tau \eta_{u_\vx}(\vx) + \tau \sum\limits_{j \neq u_\vx, j\neq i} \eta_j(\vx)\\
    \Longleftrightarrow [1-(N_c-1)\tau] \eta_{u_\vx}(\vx)  + \tau \sum\limits_{j \neq u_\vx} \eta_j(\vx) & \geq [1-(N_c-1)\tau] \eta_{i}(\vx) + \tau \sum\limits_{j\neq i} \eta_j(\vx)\\ \Longleftrightarrow
    \sum\limits_{j \in [Nc]} \tau_{j,u_\vx} \eta_j(\vx) &\geq \sum\limits_{j\in[N_c]}\tau_{ji}\eta_i(\vx)\\
    \Longleftrightarrow \widetilde{\eta}_{u_\vx}(\vx) &\geq \widetilde{\eta}_i(\vx)
    \end{align*}
    
    Since $\widetilde{\eta}_{u_\vx}(\vx) \geq \widetilde{\eta}_{s_\vx}(\vx) + 2\epsilon$, then $\widetilde{\eta}_{u_\vx}(\vx) - \epsilon \geq \widetilde{\eta}_i(\vx) + \epsilon$ and thus $f_{u_\vx} \geq f_i(\vx)$ $\forall i \in [N_c]$, which implies $f_{m_\vx}(\vx) = f_{u_\vx}(\vx)$. As a result, second term in (\ref{lemma:part1}) and second term in (\ref{lemma1:accepte}) will be 0.
    \end{proof}
    
\end{enumerate}

\begin{theorem}
\label{theorem3}
Assume $\eta$ and $f$ satisfy the same conditions as Lemma~\ref{lemma1}. Also assume $\xi < \delta$ and further assume that $\epsilon \leq \min\left(\frac{t_0 \delta^2 \min\limits_i\tau_{ii}-\xi^2-\xi}{\delta^2}, (t_0-\xi)\min\limits_i \tau_{ii}\right)$. Let $\ynoise_{new}$ be the output of the \LRTCorr~with $(\vx,\ynoise)$, $f$, and the approximate $\hat{\delta}$. Then:
\begin{enumerate}
\item{Sensitivity Optimized Critical Value.} Let $\delta =  \min\limits_\vx \left[ \frac{\tau_{\ynoise, \ynoise}\eta_{s_\vx}(\vx)+\sum\limits_{j\neq \ynoise}\tau_{j,\ynoise}\eta_j(\vx)}{f_{m_\vx}(\vx)}\right]$ then :
\[
\Pr_{(x,y)\sim D} \left[ \ynoise_{new} \neq {h^*(\vx)} , \ynoise \text{ is rejected}\right] \leq C\left[O(\max(\epsilon, \xi))\right]^\lambda + \pr\left[u_\vx \neq m_\vx, u_\vx \neq \ynoise\right]
\]
\item{Specificity Optimized Critical Value.} Let $\delta = \max\limits_\vx \frac{f_{\ynoise}(\vx)}{\tau_{m_\vx, m_\vx}\eta_{s_\vx}(\vx)+\sum\limits_{j\neq m_\vx}\tau_{j,m_\vx}\eta_j(\vx)}$ then :
\[
\Pr_{(x,y)\sim D} \left[ \ynoise_{new} \neq h^*(\vx) , \ynoise \text{ is accepted} \right] \leq C\left[O(\max(\epsilon, \xi))\right]^\lambda + \pr\left[u_\vx \neq m_\vx, u_\vx \neq \ynoise\right]
\]
\end{enumerate}
\end{theorem}

\begin{proof}
The proof will be similar to the proof of Lemma~\ref{lemma1}, but we need to adjust the error introduced by picking $\hat{\delta}$. Recall that $\xi$ and  $\epsilon $ are both less than one.

If we pick $\hat{\delta}$ instead of $\delta$, then for (\ref{lemma1:2.1.1}) in Lemma~\ref{lemma1}, we have:

\begin{align}
\label{theorem3:part1}
&\pr\left[h^*(\vx) = \ynoise, \frac{f_{\ynoise}(\vx)}{f_{m_\vx}(\vx)} < \hat{\delta}\right] 
= \pr\left[h^*(\vx) = \ynoise, {f_{\ynoise}(\vx)} < \hat{\delta}{f_{m_\vx}(\vx)}\right] \nonumber \\
&\leq \pr\left[\eta_{\ynoise}(\vx) \geq \eta_{s_\vx}(\vx), \widetilde{\eta}_{\ynoise}(\vx) - \epsilon < \hat{\delta} f_{m_\vx}(\vx) \right]  \nonumber \\
&\leq \pr\left[\eta_{s_\vx}(\vx) \leq \eta_{\ynoise}(\vx)  < \frac{\hat{\delta} f_{m_\vx}(\vx) - \sum\limits_{j\neq \ynoise}\tau_{j,\ynoise}\eta_j(\vx)}{\tau_{\ynoise, \ynoise}} + \frac{\epsilon}{\tau_{\ynoise, \ynoise}} \right] \nonumber\\
&\leq \pr\left[\eta_{s_\vx}(\vx) \leq \eta_{\ynoise}(\vx)  < \frac{(\delta+\xi) f_{m_\vx}(\vx) - \sum\limits_{j\neq \ynoise}\tau_{j,\ynoise}\eta_j(\vx)}{\tau_{\ynoise, \ynoise}} + \frac{\epsilon}{\tau_{\ynoise, \ynoise}} \right] \nonumber \\
&\leq \pr\left[\eta_{s_\vx}(\vx) \leq \eta_{\ynoise}(\vx)  < \frac{\delta f_{m_\vx}(\vx) - \sum\limits_{j\neq \ynoise}\tau_{j,\ynoise}\eta_j(\vx)}{\tau_{\ynoise, \ynoise}} + \frac{\epsilon+\xi}{\tau_{\ynoise, \ynoise}} \right] \nonumber\\
&\leq C\left[\frac{\epsilon+\xi}{\tau_{\ynoise, \ynoise}}\right]^\lambda
\end{align}

The same upper bound holds for (\ref{lemma1:2.2}) with the same reason. Then:
\begin{align*}
&\pr\left[\widetilde{y}_{new} \neq h^*(\vx) , \ynoise \text{ is rejected}\right] \leq (\ref{theorem3:part1}) + \pr\left[u_\vx \neq m_\vx, u_\vx \neq \ynoise\right]\\
&= C\left[O(\max(\epsilon, \xi))\right]^\lambda + \pr\left[u_\vx \neq m_\vx, u_\vx \neq \ynoise\right]
\end{align*}

We next analyze (\ref{lemma1:accepte:1}) in Lemma~\ref{lemma1}:

\begin{align}
&\pr\left[\eta_{m_\vx}(\vx) \geq \eta_{s_\vx}(\vx), f_{m_\vx}(\vx) \leq f_{\ynoise}(\vx)/\hat{\delta}\right] \leq 
\pr\left[\eta_{m_\vx}(\vx) \geq \eta_{s_\vx}(\vx), \widetilde{\eta}_{m_\vx}(\vx) - \epsilon \leq f_{\ynoise}(\vx)/\hat{\delta}\right] \nonumber\\
&= \pr\left[\eta_{s_\vx}(\vx) \leq \eta_{m_\vx}(\vx) \leq \frac{f_{\ynoise}(\vx)/\hat{\delta} -\sum\limits_{j\neq m_\vx}\tau_{j,m_\vx}\eta_j(\vx)}{\tau_{m_\vx, m_\vx}} + \frac{\epsilon}{\tau_{m_\vx, m_\vx}}\right] \nonumber\\
&\leq \pr\left[\eta_{s_\vx}(\vx) \leq \eta_{m_\vx}(\vx) \leq \frac{f_{\ynoise}(\vx)/(\delta-\xi) - \sum\limits_{j\neq m_\vx}\tau_{j,m_\vx}\eta_j(\vx)}{\tau_{m_\vx, m_\vx}} + \frac{\epsilon}{\tau_{m_\vx, m_\vx}} \right] \nonumber\\
&=\pr\left[0 < \eta_{m_\vx}(\vx) - \eta_{s_\vx}(\vx) < \frac{f_{\ynoise}(\vx )/\delta - \sum\limits_{j\neq m_\vx}\tau_{j,m_\vx}\eta_j(\vx)}{\tau_{m_\vx, m_\vx}} - \eta_{s_\vx}(\vx) + \frac{\epsilon}{\tau_{m_\vx, m_\vx}} + \frac{\frac{\xi f_{\ynoise}(\vx)}{\delta(\delta-\xi)}}{\tau_{m_\vx, m_\vx}}\right] \nonumber
\end{align}

Observe that $\frac{\xi}{\delta(\delta-\xi)} = \frac{\delta}{(\delta-\xi)}\frac{\xi}{\delta^2} = \left[1 + O(\xi)\right]\frac{\xi}{\delta^2}$, where second equality comes from Taylor expansion. Then we substitute the $\delta$ as what we did in Lemma~\ref{lemma1} and continue the calculation:

\begin{align}
\label{theorem3:part2}
&\pr\left[\eta_{m_\vx}(\vx) \geq \eta_{s_\vx}(\vx), f_{m_\vx}(\vx) \leq f_{\ynoise}(\vx)/\hat{\delta}\right] \nonumber \\
&\leq \pr\left[0 < \eta_{m_\vx}(\vx) - \eta_{s_\vx}(\vx) < \frac{f_{\ynoise}(\vx )/\delta - \sum\limits_{j\neq m_\vx}\tau_{j,m_\vx}\eta_j(\vx)}{\tau_{m_\vx, m_\vx}} - \eta_{s_\vx}(\vx) + \frac{\epsilon}{\tau_{m_\vx, m_\vx}} + \frac{\frac{\xi f_{\ynoise}(\vx)}{\delta(\delta-\xi)}}{\tau_{m_\vx, m_\vx}}\right] \nonumber\\
&\leq \pr\left[0 \leq \eta_{m_\vx}(\vx) - \eta_{s_\vx}(\vx) \leq  \frac{\epsilon}{\tau_{m_\vx, m_\vx}} + \frac{\xi f_{\ynoise}(\vx)}{\delta^2 \tau_{m_\vx, m_\vx}} + \frac{\xi O(\xi) f_{\ynoise}(\vx)}{\delta^2 \tau_{m_\vx, m_\vx}}\right] \nonumber\\
&\leq \pr\left[0 \leq \eta_{m_\vx}(\vx) - \eta_{s_\vx}(\vx) \leq  \frac{\epsilon}{\tau_{m_\vx, m_\vx}} + \frac{\xi}{\delta^2 \tau_{m_\vx, m_\vx}} + \frac{\xi^2}{\delta^2 \tau_{m_\vx, m_\vx}}\right] \nonumber\\
&\leq C\left[\frac{\epsilon}{\tau_{m_\vx, m_\vx}} + \frac{\xi}{\delta^2 \tau_{m_\vx, m_\vx}} + \frac{\xi^2}{\delta^2 \tau_{m_\vx, m_\vx}}\right]^\lambda
\end{align}

Here Tsybakove condition hold, because $\frac{\epsilon}{\tau_{m_\vx, m_\vx}} + \frac{\xi}{\delta^2 \tau_{m_\vx, m_\vx}} + \frac{\xi^2}{\delta^2 \tau_{m_\vx, m_\vx}} \leq \frac{t_0 \delta^2 \min\limits_i \tau_{ii}-\xi^2-\xi}{\delta^2\tau_{m_\vx, m_\vx}} + \frac{\xi}{\delta^2 \tau_{m_\vx, m_\vx}} + \frac{\xi^2}{\delta^2 \tau_{m_\vx, m_\vx}} \leq t_0$. As a result:

\begin{align*}
&\pr\left[\ynoise_{new}\neq h^*(\vx), \ynoise \text{ is accepted}\right] \\
&\leq (\ref{theorem3:part2}) + \pr\left[u_\vx \neq m_\vx, u_\vx \neq \ynoise\right] \\
&\leq C\left[O(\max(\epsilon, \xi))\right]^\lambda + \pr\left[u_\vx \neq m_\vx, u_\vx \neq \ynoise\right]\\
\end{align*}
which compete the proof for cases that are accepted.

Other terms will not be affected by the choice of $\delta$. By now we completes the proof. 
\end{proof}